\pdfoutput=1
\documentclass{article} %
\usepackage{iclr2027_conference,times}

\newif\ifarxiv
\arxivtrue
\ifarxiv\iclrfinalcopy\fi

\usepackage[ruled,vlined]{algorithm2e}
\usepackage{caption}
\usepackage{subcaption}
\usepackage{graphicx}
\usepackage{float}
\usepackage{todonotes}
\usepackage{amsmath, amsthm}
\newtheorem{lemma}{Lemma}
\newcommand{\enc}{\operatorname{enc}}
\newcommand{\dec}{\operatorname{dec}}

\usepackage[utf8]{inputenc} %
\usepackage[T1]{fontenc}    %
\usepackage{hyperref}       %
\usepackage{url}            %
\usepackage{booktabs}       %
\usepackage{amsfonts}       %
\usepackage{nicefrac}       %
\usepackage{microtype}      %
\usepackage{xcolor}         %

\definecolor{tabBlue}{RGB}{31,119,180}
\definecolor{tabOrange}{RGB}{255,127,14}
\definecolor{tabRed}{RGB}{214,39,40}
\definecolor{tabGray}{RGB}{127,127,127}

\hypersetup{colorlinks=true, citecolor=blue, linkcolor=blue, urlcolor=blue}

\usepackage{amsmath,amsfonts,bm}

\def\eqref#1{equation~\ref{#1}}

\def\1{\bm{1}}

\def\vc{{\bm{c}}}

\def\ve{{\bm{e}}}

\def\vg{{\bm{g}}}

\def\vk{{\bm{k}}}

\def\vq{{\bm{q}}}
\def\vr{{\bm{r}}}

\def\vv{{\bm{v}}}
\def\vw{{\bm{w}}}
\def\vx{{\bm{x}}}
\def\vy{{\bm{y}}}
\def\vz{{\bm{z}}}

\def\mA{{\bm{A}}}

\def\mC{{\bm{C}}}

\def\mG{{\bm{G}}}
\def\mH{{\bm{H}}}
\def\mI{{\bm{I}}}

\def\mM{{\bm{M}}}

\def\mP{{\bm{P}}}

\def\mS{{\bm{S}}}

\def\mW{{\bm{W}}}

\DeclareMathAlphabet{\mathsfit}{\encodingdefault}{\sfdefault}{m}{sl}
\SetMathAlphabet{\mathsfit}{bold}{\encodingdefault}{\sfdefault}{bx}{n}

\newcommand{\R}{\mathbb{R}}

\DeclareMathOperator*{\argmin}{arg\,min}

\DeclareMathOperator{\Tr}{Tr}

\title{DeltaNet-PC: Test-Time Memory Without Backpropagation Through Time}

\author{Todd Morrill \\
Computer Science, Columbia University \\
New York, NY, USA \\
\texttt{todd@cs.columbia.edu}
\And
Richard Zemel \\
Computer Science, Columbia University \\
New York, NY, USA \\
\texttt{zemel@cs.columbia.edu}
}

\begin{document}

\maketitle
\ifarxiv
\lhead{}
\renewcommand{\headrulewidth}{0pt}
\fi

\begin{abstract}
Test-time memorization of sequential information is a foundational capability for any intelligent system that needs to rapidly store and retrieve data, as demonstrated by the attention mechanism in large language models. However, Transformers and linear attention variants are trained via backpropagation through time (BPTT), which has drawbacks such as its enormous memory requirements.
Predictive coding (PC)---an algorithm that optimizes a prediction-error energy function with local updates---offers a compelling solution but still lacks test-time memory for sequential data. We develop a solution to this problem by treating the per-step memory update $\Delta\mS_t$ as a latent variable, as opposed to a slow weight to be learned. By taking the log probability of a standard PC-based probabilistic model, we derive a PC-native solution for verbatim sequence memory, and relate our update rule to common choices in the linear attention literature (e.g., DeltaNet). 
We present an extensive set of empirical tests to probe our system's properties, including its scaling capacity, memory footprint, and ability to infer a context from evidence (e.g., bank (river) vs. bank (money)). We show how DeltaNet-PC can be used for storing single sequences of data, multiple sequences of data, and key-value associative memory pairs with perfect recall accuracy.
\end{abstract}

\section{Introduction}
\label{sec:intro}
Test-time memorization of sequential information is a foundational capability for intelligent biological and artificial systems. Much of the success of modern large language models (LLMs) can be attributed to test-time adaptivity \citep{vonoswald2023}, in particular the self-attention mechanism found in Transformers \citep{vaswani2017}, and more recently the linear attention mechanism found in hybrid sequence models \citep{katharopoulos2020,schlag2021,sun2023,yang2024,kimiteam2026}.

\begin{figure*}
\centering
\includegraphics[width=\textwidth]{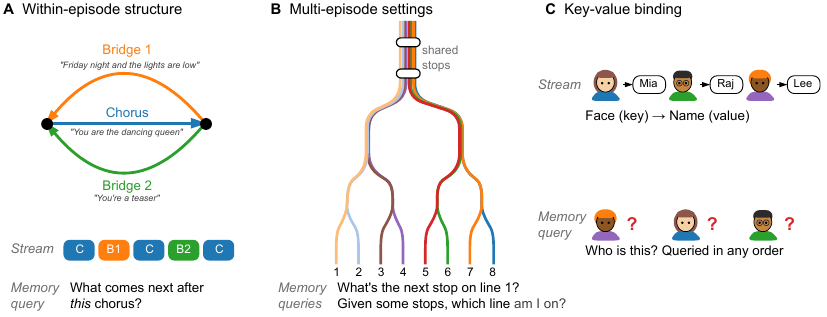}
\caption{We evaluate our system on several sequential memory tasks. \textbf{A:} Within-episode structure tests the model's ability to handle non-Markovian dynamics in a single sequence. \textbf{B:} Multi-episode settings test the system's ability to memorize and reason about multiple sequences, particularly with highly overlapping subsequences. \textbf{C:} Key-value binding tests whether the system can retrieve a value when cued with a key.}
\label{fig:episodic_capabilities}
\end{figure*}

However, these systems are trained using backpropagation through time (BPTT) \citep{werbos1990}, which has some well-known drawbacks: memory that grows with the sequence length \citep{marschall2020}, and the inability to learn in streaming scenarios without locking the system to wait for a backward pass to finish \citep{ororbia2023a}.
This motivates an algorithm that relies on updates that are local in both space and time.

Predictive coding is one of the most successful backpropagation alternatives \citep{rao1999,whittington2017,pinchetti2025}, and its parameter updates are local in space (i.e., Hebbian, only requiring pre- and postsynaptic activity) \citep{bogacz2017}.
Prior PC-based works have developed methods for sequential modeling \citep{tang2023} and continual updates \citep{yoo2022}, though they lack what is arguably one of the most important parts of a Transformer: some notion of test-time memorization for sequential data. Self-attention, however, uses an unbounded key-value cache---extremely taxing for computers and a resource that brains seem to succeed without---which motivates our focus on a fixed-size memory.

In brief, we are storing pairs: a key $\vk_t$ and a value $\vv_t$, which are bound together in a memory $\mS$ at time $t$.
We connect predictive coding to the linear attention literature (e.g., DeltaNet \citep{schlag2021}) through a key insight: once you treat matrix-valued updates $\Delta\mS_t$ to the memory as latent variables, as opposed to slowly learned parameters, predictive coding natively provides linear attention machinery.

Our tasks include three episodic memory scenarios, such as memorizing a song, where the chorus is followed by different bridges, requiring non-Markovian memory (Figure~\ref{fig:episodic_capabilities}A), and putting names to faces (key-value binding, Figure~\ref{fig:episodic_capabilities}C). Our system also faithfully stores memories with a high degree of overlap, like a metro map, and can infer the line given some observed stops (Figure~\ref{fig:episodic_capabilities}B). We evaluate on community benchmarks such as MQAR \citep{arora2024}, and image-based storage and retrieval tasks.

Mainstream BPTT architectures (e.g., Transformers, SSMs, etc.) solve verbatim recall easily and we claim no accuracy advantage over them, but they violate our locality and/or fixed-size memory constraints. We compare to fixed memory capacity baselines (i.e., not Transformers, whose key-value cache is unbounded), including a BPTT-trained DeltaNet model \citep{yang2024c}, and several outer-product associative memory networks.

Our contributions are as follows: (i) a sequential PC model that binds new memories at test time, which uses updates that are local in both space and time and has a fixed-size memory footprint, (ii) a write rule derived from treating the memory update as an inferred latent variable, which can be related to common choices in the linear attention literature, and (iii) empirical results demonstrating our system's ability to match the perfect recall accuracy of BPTT-trained linear attention models at a fraction of the training memory cost, while overcoming the limitations of models that can't learn (e.g., static outer-product models). 

\section{Background \& Related Work}
\label{sec:background}

\subsection{Predictive coding}
Predictive coding (PC) is an algorithm for \emph{inference} and \emph{learning} that optimizes a prediction-error energy function \citep{rao1999,whittington2017,pinchetti2025}.
In its most common form, each layer of the network has a latent variable $\vz$ that is optimized to minimize a global energy function---a process called inference (also referred to algorithmically as settling). Once the latent variables have settled, the weights are updated to minimize the same energy function---learning. A standard PC assumption is that each layer's latent is a Gaussian random variable with a mean given by the output of the previous layer. This allows one to write down a probabilistic model and derive a resulting energy function from the log probability. For example, the probability distribution, and resulting energy function, for an $\ell$-layer feedforward network with latents $(\vz_0=\vx), \vz_1, \ldots, \vz_{\ell}$, weights $\mW_1, \ldots, \mW_{\ell}$, and activation function $g$ are
\begin{gather}
\label{eq:pc_model}
p(\vz_0, \vz_1, \ldots, \vz_{\ell}) = \prod_{i=1}^{\ell} p(\vz_i \mid \vz_{i-1}),
\quad
p(\vz_i \mid \vz_{i-1}) = \mathcal{N}(\vz_i; \mW_i g(\vz_{i-1}), \mathbf{I}), \nonumber \\
\label{eq:pc_energy}
E(\vz_0, \vz_1, \ldots, \vz_{\ell}) = -\log p(\vz_0, \vz_1, \ldots, \vz_{\ell}) = \frac12\sum_{i=1}^{\ell} \lVert \vz_i - \mW_i g(\vz_{i-1})\rVert^2.
\end{gather}

\citet{tang2023} developed tPC for sequential modeling, which can memorize \emph{training} data in the slow weights of their network, but they do not treat the per-step update as a latent variable so their model does not perform test-time memorization. \citet{yoo2022} developed BayesPCN, a PC-based continual learning system, but they do not address sequential data.

\subsection{Linear attention}
Linear attention is a mechanism for performing attention with a constant memory footprint \citep{katharopoulos2020}, in contrast to the key-value cache that grows linearly in the sequence length for standard attention \citep{vaswani2017}. It was initially conceptualized as a sum of outer product memories \citep{katharopoulos2020} given by $\mS_t = \mS_{t-1} +  \vv_t \vk_t^{\top}$, where $\mS$ is a linear memory, $\vv_t$ is a value, and $\vk_t$ is a key. The readout of the memory is given by $\mS \vq_t$, where $\vq_t$ is a query. This style of memory has fixed capacity---proportional to the dimensionality of the key space---and is prone to interference, which motivated DeltaNet \mbox{\citep{schlag2021}}, with a learnable write gate $\beta_t$ and an erase term to reduce interference: $\mS_t = \mS_{t-1} + \beta_t(\vv_t - \mS_{t-1}\vk_t)\vk_t^{\top}$. This is one gradient step on $\mathcal{L}_t(\mS) = \frac12\lVert \vv_t - \mS\vk_t\rVert^2$ from $\mS_{t-1}$ with step size $\beta_t$---the least mean squares (LMS) rule \citep{widrow1960}. With L2-normalized keys \citep{yang2024c} it is also normalized LMS (NLMS), which divides the step by $\lVert\vk_t\rVert^2$.

More recent work has continued to iterate on this core idea, including the use of forgetting mechanisms \citep{yang2024}, more complex key-value bindings \citep{sun2025b}, and the use of recursive least squares (RLS) for key-value bindings \citep{vonoswald2024}. Table~\ref{tab:related_work} situates DeltaNet-PC in the context of prior works.

\providecommand{\cmk}{\textcolor{green}{\checkmark}}
\providecommand{\xmk}{\textcolor{red}{$\times$}}
\providecommand{\nak}{--}
\begin{table}[h]
\centering
\small
\setlength{\tabcolsep}{4pt}
\caption{Situating DeltaNet-PC. Memory is \emph{one layer's} footprint; $L$ the sequence length and $D$ the latent space  dimensionality. Local: weight updates use only pre- and postsynaptic activity (space) and only the current step (time). Sequential: stores ordered sequences. Trained: the system learns the keys and values, rather than given as raw data. Test-time: an episode is stored in one pass with no gradient steps on the weights. See Appendix~\ref{app:related} for an extended related work section and Appendix~\ref{app:linear_attention}, which relates our work to linear attention.}
\label{tab:related_work}
\resizebox{\ifdim\width>\textwidth\textwidth\else\width\fi}{!}{%
\begin{tabular}{@{}llccccccc@{}}
\toprule
& & Training & Inference & Local & Local & Sequential & Trained & Test-time \\
Model & Write rule & memory & memory & in space & in time & data & keys, values & write \\
\midrule
Transformer \citep{vaswani2017} & append (KV cache) & $LD$ & $LD$ & \xmk & \xmk & \cmk & \cmk & \cmk \\
Linear attention \citep{katharopoulos2020} & outer product & $LD$ & $D^2$ & \xmk & \xmk & \cmk & \cmk & \cmk \\
DeltaNet \citep{schlag2021,yang2024c} & delta rule (LMS/NLMS) & $LD$ & $D^2$ & \xmk & \xmk & \cmk & \cmk & \cmk \\
Mesa layer \citep{vonoswald2024,vonoswald2026} & RLS & $LD$ & $D^2$ & \xmk & \xmk & \cmk & \cmk & \cmk \\
Lin. assoc. memory \citep{kohonen1972,kohonen1974} & outer product / RLS & \nak & $D^2$ & \cmk & \nak & \xmk & \xmk & \cmk \\
tPC \citep{tang2023} & delta rule (LMS) & $D^2$ & $D^2$ & \cmk & \cmk & \cmk & \cmk & \xmk \\
BayesPCN \citep{yoo2022} & RLS (Bayesian) & \nak & $D^2$ & \cmk & \nak & \xmk & \xmk & \cmk \\
DeltaNet-PC (ours) & RLS (inferred $\Delta\mS_t$) & $D^2$ & $D^2$ & \cmk & \cmk & \cmk & \cmk & \cmk \\
\bottomrule
\end{tabular}}
\end{table}

\section{Methods}
\label{sec:methods}
We first define our probabilistic model and then derive its resulting energy function. We then describe the complete architecture and training procedure.

\begin{figure*}
\centering
\includegraphics[width=\textwidth]{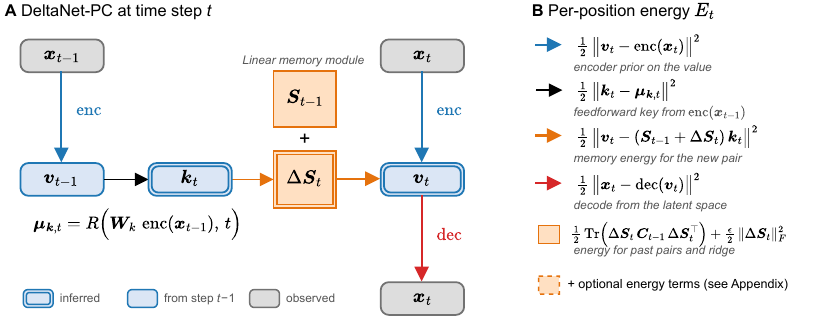}
\caption{\textbf{A:} DeltaNet-PC architecture. \textbf{B:} The per-position energy $E_t$. See Appendix~\ref{app:full_energy_fn} for the full energy function.}
\label{fig:architecture}
\end{figure*}

\subsection{Memory module}
\paragraph{Joint distribution.} We specify a standard PC-native probabilistic graphical model, which results in the memory write---recursive least squares \citep{plackett1950} when the memory update $\Delta\mS_t$ is a fully settled latent. We assume a linear memory store $\mS \in \R^{D \times D}$ and we would like this memory to encode transitions at time $t$ from keys to values such that $\vv_t \approx \mS_t \vk_t$ for a key-value pair $(\vk_t, \vv_t)$. Importantly, online updates to the memory $\mS$ should disturb what past keys retrieve as little as possible. Our joint distribution restricted to the memory model over a sequence of length $L$ is given by
\begin{gather}
    p(\vk_{1:L}, \vv_{1:L}, \mS) = p(\mS)\prod_{t=1}^{L} p(\vv_t \mid \vk_t, \mS) p(\vk_t \mid \cdot), \label{eq:joint_dist} \\
    p(\mS) \;=\; \mathcal{N}(\mS; \mathbf{0}, \mathbf{I}/\epsilon), \label{eq:memory_joint}
    \qquad
    p(\vv_t \mid \vk_t, \mS) \;=\; \mathcal{N}\bigl(\vv_t;\, \mS\vk_t,\, \mathbf{I}\bigr),
    \qquad 
    p(\vk_t \mid \cdot) \;=\; \mathcal{N}\bigl(\vk_t;\, \cdot,\, \mathbf{I}\bigr). \nonumber
\end{gather}
This allows us to make the standard predictive coding assumption that latent variables are Gaussian distributed with a mean given by the output of the previous layer.\footnote{See Section~\ref{sec:arch} for a definition of $\vk_t$ and the full joint distribution.}

\paragraph{Energy function.} Taking the negative log of Equation~\ref{eq:joint_dist} and restricting our attention to terms involving the memory $\mS$ gives the energy function for the memory module:
\begin{gather}
E(\mS) = -\log p(\vk_{1:L}, \vv_{1:L}, \mS) \;=\; \frac{\epsilon}{2}\lVert \mS\rVert^2_F + \sum_{t=1}^{L}\tfrac12\bigl\lVert \vv_t -  \mS\vk_t\bigr\rVert^2 + \text{const.}, \label{eq:memory_energy}
\end{gather}
which is equivalent to ridge regression on all $L$ pairs $(\vk_t, \vv_t)$. However, as written, this is a batched objective that requires access to all $L$ pairs---incompatible with our online learning scenario, where data points are observed one-at-a-time. Furthermore, our fixed-size memory footprint precludes storing the full dataset for offline learning. Past latents $(\vk_i, \vv_i)_{i<t}$ are committed once they have been inferred and are not revisited. We define a partial energy function that accounts for all observed data up to time $t$
\begin{gather}
    E_t(\mS) = E_{t-1}(\mS) + \tfrac12\lVert \vv_t -  \mS\vk_t\rVert^2, \label{eq:partial_energy} \\
    E_0(\mS) = \frac{\epsilon}{2}\lVert \mS\rVert^2_F, \qquad E_L(\mS) = E(\mS), \nonumber \\
    \mS_t := \argmin_{\mS} E_t(\mS), \qquad \mS_0 := \boldsymbol{0}, \qquad \Delta\mS_t := \mS_t - \mS_{t-1}, \nonumber
\end{gather}
which makes clear that at each time step, we would like to update the memory $\mS_{t-1}$ by an amount $\Delta\mS_t$ that minimizes the energy function $E_t(\mS)$.

We now express $E_t$ as a function of the latent variable $\Delta\mS_t$ with the past summarized in ($\mS_{t-1}, \mC_{t-1}$), where $\mC_{t-1} = \sum_{i < t} \vk_i \vk_i^{\top}$ is a summary of the past keys, so that we can maintain a fixed memory footprint that only involves two $D \times D$ matrices that don't grow with the sequence length.

Lemma~\ref{lem:taylor_expansion} Taylor expands $E_{t-1}$ about $\mS_{t-1}$. Let $\vr_t = \vv_t - \mS_{t-1}\vk_t$.
\begin{lemma}
\label{lem:taylor_expansion}
For any $\Delta$
\begin{equation*}
    E_t(\mS_{t-1} + \Delta) = E_{t-1}(\mS_{t-1}) + \langle\nabla E_{t-1}(\mS_{t-1}), \Delta \rangle + \tfrac12 \Tr(\Delta(\epsilon\mathbf{I} + \mC_{t-1})\Delta^{\top}) +
    \tfrac12\lVert \vr_t - \Delta\vk_t\rVert^2.
\end{equation*}
\end{lemma}
Appendix~\ref{app:full_energy} gives the proof and the calculations that lead to the closed form
\begin{equation}
    \label{eq:energy_t}
    E_t(\mS_{t-1} + \Delta) = \underbrace{\frac12\lVert \vr_t - \Delta\vk_t\rVert^2}_{\text{new pair}} + \underbrace{\frac12\Tr(\Delta\mC_{t-1}\Delta^{\top})}_{\text{stored pairs}} + \underbrace{\frac{\epsilon}{2}\lVert \Delta \rVert_F^2}_{\text{ridge}} + \text{ const.}
\end{equation}
Recall that $\mS_t$ is defined as the minimizer of $E_t(\mS)$, which means that $\Delta\mS_t = \mS_t - \mS_{t-1}$ is the minimizer of $E_t(\mS_{t-1} + \Delta)$. Notably, we are treating $\Delta\mS_t$ as a latent variable, which means that we seek the \emph{minimizer} of $E_t(\mS_{t-1} + \Delta)$ with respect to $\Delta$, not simply a small gradient step on the energy function, which would be the case if we were treating $\mS$ as a learned weight. So we set the gradient of $E_t(\mS_{t-1} + \Delta)$ with respect to $\Delta$ to zero, which gives the update
\begin{equation}
\label{eq:delta_s_t_update}
\Delta\mS_t = \vr_t\vk_t^{\top}(\vk_t\vk_t^{\top} + \mA_{t-1})^{-1}
\end{equation}
where $\mA_{t-1} = \mC_{t-1} + \epsilon \mathbf{I}$. Derivations are provided in Appendix~\ref{app:full_energy}.

\paragraph{Online updates.}
Let us denote $\mP := \mA^{-1}$, $s_t := \vk_t^{\top}\mP_{t-1}\vk_t$, and $g_t := \frac{1}{1+s_t}$. In Appendix~\ref{app:full_energy} we show how at each time step we update $\mS_t$ and $\mP_t$ as
\begin{gather}
    \label{eq:rls_update}
    \mS_t = \mS_{t-1} + g_t \vr_t \left(\mP_{t-1} \vk_t\right)^{\top}, \qquad \mS_0 = \mathbf{0} \\
    \mP_t = \mP_{t-1} - g_t (\mP_{t-1}\vk_t)(\mP_{t-1}\vk_t)^{\top}, \qquad \mP_0 = \mathbf{I}/\epsilon.
\end{gather}
These updates are equivalent to the recursive least squares (RLS) algorithm \citep{plackett1950}. Importantly, this update rule no longer requires an inverse to be computed at every time step, which makes it feasible for online learning scenarios. It only stores $\mS_{t-1}$ and $\mP_{t-1}$, two $D \times D$ matrices that do not grow with the sequence length $L$, so it complies with our fixed-size memory footprint.

\paragraph{Interpretation.} Equations~\ref{eq:energy_t} and \ref{eq:rls_update} can be interpreted relative to common choices that have been made in the linear attention literature, in particular (i) the outer-product form found in \citep{katharopoulos2020}, (ii) the erase term found in \citep{schlag2021}, (iii) the write strength (e.g., the gate $\beta_t$ that \citet{schlag2021} must learn), and (iv) the least-squares objective over all past pairs, which the Mesa layer solves with RLS \citep{vonoswald2024}, all of which follow from our energy function. See Appendix~\ref{app:linear_attention} for a complete mathematical mapping.

\subsection{DeltaNet-PC architecture}
\label{sec:arch}
The memory module is embedded in the following architecture (Figure~\ref{fig:architecture}):
\begin{gather}
    p(\vx_{0:L}, \vv_{1:L}, \vk_{1:L}, \mS) \propto p(\mS)\prod_{t=1}^{L} p(\vx_t \mid \vv_t)\, p(\vv_t \mid \vk_t, \mS)\, p(\vv_t \mid \vx_t)\, p(\vk_t \mid \vx_{t-1}), \label{eq:full_joint} \\
    p(\vx_t \mid \vv_t) = \mathcal{N}\bigl(\vx_t;\, \dec(\vv_t),\, \mathbf{I}\bigr),
    \qquad
    p(\vv_t \mid \vx_t) = \mathcal{N}(\vv_t;\, \enc(\vx_t), \mathbf{I}),\nonumber\\
    p(\vk_t \mid \vx_{t-1}) = \mathcal{N}\bigl(\vk_t;\, R(\mW_k\,\enc(\vx_{t-1}), t),\, \mathbf{I}\bigr). \nonumber
\end{gather}
We use an encoder $\enc(\vx)$ to map observations $\vx \in \R^{d_{\text{obs}}}$ to a latent space, which is parameterized by a linear layer $\mW_{\enc} \in \R^{D \times d_{\text{obs}}}$ with no bias term. The decoder $\dec(\vv)$ (with bias) maps back to observations. We use LayerNorm at the input to the decoder as our activation function, writing $g(\cdot) := \mathrm{LayerNorm}(\cdot)$. $\dec(\vv_t) = \mW_{\dec}g(\vv_t)$ with $\mW_{\dec} \in \R^{d_{\text{obs}} \times D}$. The key map is parameterized by a linear layer $\mW_k \in \R^{D \times D}$. A rotary position code $R(\cdot, t)$ with learned per-band frequencies (rotary positional encoding (RoPE) at initialization \citep{su2024}) is used to represent the position of the key in the sequence.\footnote{Crucially, we do not hardcode any inductive bias about the task via RoPE frequencies---the model must \emph{learn} whether ordering matters for the task (e.g., single episode tasks) or not (e.g., key-value binding tasks). MQAR is the one exception, where RoPE is off for a direct comparison with the DeltaNet baseline.} We represent context in multi-episode settings with $R(\cdot, t, \vc)$, where $\vc$ is the context (see Appendix~\ref{app:rope}). In our model, queries are defined as keys $\vq_t = \vk_t$.\footnote{$\vv_t$ receives two Gaussian predictions so the joint is a product of experts and is defined up to a normalization constant that does not affect inference.}

Note that whenever we define a Gaussian-distributed variable like the ones above, there is a corresponding energy term of the form $\tfrac12\lVert \text{variable} - \text{mean}\rVert^2$ that is added to the total energy of the system. Figure~\ref{fig:architecture}B displays the per-position energy terms and we specify the complete energy function in Appendix~\ref{app:full_energy_fn}.

While we use a single linear layer for $\enc(\vx)$ and $\dec(\vv)$ in our experiments, the formulation extends straightforwardly to multi-layer encoders and decoders with interior PC latents, and other layer types (e.g., CNNs).

\subsection{Training procedure}
\label{sec:training}
We provide a complete training procedure in Algorithm~\ref{alg:training} and summarize the key steps here.
\paragraph{Initialization.} We initialize $\vk_t \gets R(\mW_k \enc(\vx_{t-1}), t)$ and $\vv_t \gets \enc(\vx_t)$ from feedforward passes.

\paragraph{Inference.} The memory write phase uses Equation~\ref{eq:rls_update} to update the memory $\mS_{t-1}$ to $\mS_t$ given the current key-value pair $(\vk_t, \vv_t)$. When the task queries the memory, $\vk_t$ and $\vv_t$ are settled with the memory held fixed so that the rest of the system's parameters in $R$, $\mW_k$, $\enc$ and $\dec$ can learn. Settling is local in time and depends only on information available at step $t$, with no unrolled computation graph across positions.

\paragraph{Learning.} After settling at each position $t$, we compute the gradient of our energy function with respect to the parameters and accumulate the gradients into a running mean. After the final position $L$, we update the parameters once per sequence.

\section{Experiments}
\label{sec:experiments}

We evaluate DeltaNet-PC on three episodic memory tasks (Figure~\ref{fig:episodic_capabilities} and Appendix~\ref{app:datasets}). Training details are in Appendix~\ref{app:training} and hyperparameters are in Appendix~\ref{app:hparams}.

\textbf{Evaluation.} Examples are drawn from a mix of token-based tasks and image-based tasks. Each task has an event stream, where events are either observations or queries. Observations fill the memory. Queries probe what the system has learned and we evaluate the system on these events. We report token-level \emph{recall accuracy}---the fraction of correctly predicted tokens, averaged over the number of queries and test sequences. Appendices~\ref{app:eval_details} and \ref{app:ablations} give evaluation details and ablations.

\subsection{Single-episode sequential memory with non-Markovian dynamics}
The task is to store a single sequence in memory and then recall it, even when the sequence has repeated elements, i.e., non-Markovian dynamics (see Figure~\ref{fig:single_ep_seq}A). Figure~\ref{fig:episodic_capabilities}A shows an example of non-Markovian dynamics in a song, where the chorus repeats and is followed by two distinct bridges. The test set for this task is completely disjoint from the training set so we are testing the model's ability to generalize to unseen sequences at test time. Note that recall is autoregressive and \emph{not} teacher-forced: each decoded token is re-encoded to form the next key, so the system must rely on its memory to recall the whole sequence correctly.

\begin{figure}[h]
    \centering
    \includegraphics{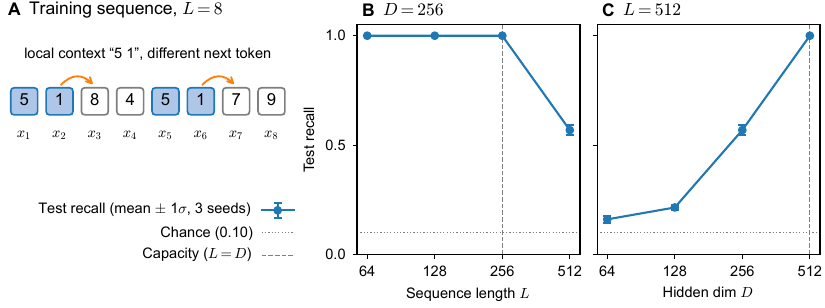}
    \caption{\textbf{A:} a sample from our single episode sequential memory task that illustrates non-Markovian dynamics. \textbf{B:} when $L\leq D$, the memory can recall sequences with perfect accuracy. \textbf{C:} scaling $D$ increases the memory's capacity allowing it to store more memories.}
    \label{fig:single_ep_seq}
\end{figure}

Figure~\ref{fig:single_ep_seq} demonstrates our system's scaling properties. The dimensionality of the memory module governs the memory capacity. In panel B, we see that when the sequence length $L$ is less than or equal to the hidden dimension $D$, the memory can recall sequences with perfect accuracy, consistent with the number of orthogonal directions in the latent space. However, when $L>D$ (e.g., $L = 512$ and $D = 256$), the memory is overloaded and we observe a graceful degradation in recall accuracy. Panel C shows that simply scaling the hidden dimension $D$ solves this problem. We tested our system up to $L=D=1{,}024$ and it maintains perfect recall accuracy as we scale (100\%, 3 seeds).

\begin{table}[ht]
\centering
\small
\caption{Test recall accuracy under various write rules of untrained and trained systems. Sequential recall at $L = D = 256$, chance $0.10$, 3 seeds, mean $\pm1\sigma$. Untrained systems cannot solve the task and even in trained systems, the write rule matters: the recursive least squares update is the only one that achieves perfect recall.}
\label{tab:write_ladder}
\resizebox{\ifdim\width>\textwidth\textwidth\else\width\fi}{!}{%
\begin{tabular}{@{}lcc@{}}
\toprule
Write rule & Untrained & Trained \\
\midrule
Outer product, $\mS \leftarrow \mS + \vv_t\vk_t^{\top}$ \citep{kohonen1972,katharopoulos2020} & 0.13 $\pm$ 0.00 & 0.66 $\pm$ 0.03 \\
Delta rule (NLMS), $\mS \leftarrow \mS + (\vv_t - \mS\vk_t)\vk_t^{\top}/\lVert\vk_t\rVert^2$ \citep{widrow1960,schlag2021} & 0.13 $\pm$ 0.00 & 0.26 $\pm$ 0.01 \\
Recursive least squares, Equation~\ref{eq:rls_update} \citep{kohonen1974,vonoswald2024} & 0.29 $\pm$ 0.01 & \textbf{1.00 $\pm$ 0.00} \\
\bottomrule
\end{tabular}}
\end{table}

We compare three popular write rules from the associative memory and linear attention literature in trained and untrained settings (Table~\ref{tab:write_ladder}), where untrained means all learned parameters are frozen at initialization. Untrained associative memory systems \citep{widrow1960,kohonen1972,kohonen1974} are insufficient for performing this task.
The write rule also has a large impact---RLS is the only write rule that achieves perfect recall (Table~\ref{tab:write_ladder}).\footnote{For completeness, we also evaluate tPC \citep{tang2023} in Appendix~\ref{app:tpc_reference}, though that system is designed for a different purpose---memorizing training sequences into slow weights, not test-time memorization.}

\subsection{Multi-episode sequential memory}
\label{sec:multiepisode}
Here the task is to store multiple context-dependent sequences in the memory at one time, where the episodes share overlapping segments. For example, consider two episodes of one data example
$$\begin{aligned}
&\text{Episode 1: } a, \textcolor{tabRed}{\boldsymbol{b}, \boldsymbol{c}}, d, e \\
&\text{Episode 2: } x, \textcolor{tabRed}{\boldsymbol{b}, \boldsymbol{c}}, f, g
\end{aligned}$$
where $\textcolor{tabRed}{\boldsymbol{b}, \boldsymbol{c}}$ is the overlapping segment with exactly the same contents at exactly the same positions. In the multi-episode setting, position and content are insufficient for recall because the next token depends on the episode context. Recall that in the multi-episode setting, keys are a function of both position $t$ and context $\vc$, $\vk_t = R(\mW_k \enc(\vx_{t-1}), t, \vc)$.

Maintaining separate but overlapping memories is commonly needed. Navigation often involves routes with overlapping segments, requiring context to correctly recall the next step of the journey (see Figure~\ref{fig:episodic_capabilities}B).

First, we evaluate our system by filling the memory with multiple episodes with overlapping segments and testing recall. Our system achieves 100\% retrieval accuracy (3 seeds) on a held-out test set---unseen sequences of multiple episodes and never-seen episode context vectors, $\vc$---with $K=8$ episodes of length $\ell = 32$ in $D = 256$ (i.e., a full memory load at $L=K\cdot \ell = 256 = D$).

Next, we probe whether the system can recover the context vector $\vc$ of an episode from observed transitions (i.e., disambiguate ``bank'' vs. ``bank'' from context clues). See Appendix~\ref{app:ctx_inference} for complete details. We create a maximally challenging test (see Figure~\ref{fig:multi_episode_examples} for a data example): the $K=8$ episodes of $\ell=32$ tokens overlap in a tree-like fashion. All $K=8$ episodes are \emph{exactly} the same in the first 8 tokens of their episodes, then 4 episodes collide in the next 8 tokens, then 2 episodes collide in the next 8 tokens, and finally all 8 episodes diverge in the last 8 tokens (i.e., $8+8+8+8=32=\ell$).

\begin{figure}[htbp]
    \centering
    \includegraphics{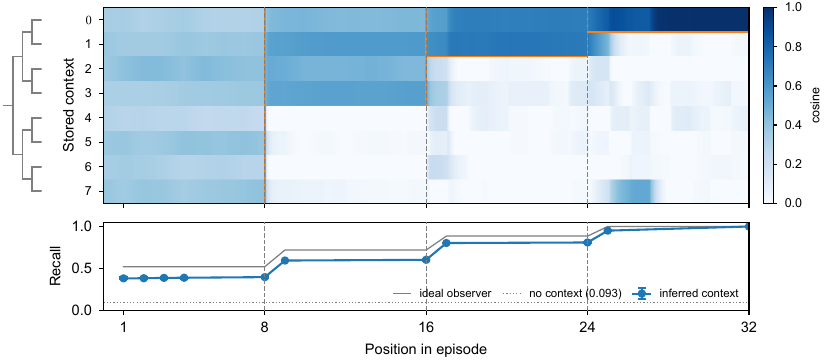}
    \caption{Contextual inference heatmap. \textbf{Top:} cosine similarity between the inferred context and the 8 possible contexts at each point in time. We observe the system resolve its uncertainty about the context as it sees more transitions. \textbf{Bottom:} given the inferred context, we measure recall accuracy for the episodes and note that it behaves nearly like the ideal observer upper bound.}
    \label{fig:context_inference_heatmap}
\end{figure}

We fill the system's memory with a data example ($K=8, \ell=32$). We then reveal observed transitions one-at-a-time and the system must infer the context $\vc$ that best explains the transitions observed so far. The inferred context stays within the span of the stored contexts, which the system tracks with a fixed-size summary (Appendix~\ref{app:ctx_inference}). Figure~\ref{fig:context_inference_heatmap} shows the results of this experiment. The heat map shows the cosine similarity between the inferred context and the 8 possible contexts at each point in time. We observe the system resolve its uncertainty about the context as it sees more transitions.\footnote{Importantly, once a transition is \emph{inconsistent} with a particular context, the cosine similarity with that context drops to near zero (see white regions in the heat map).} Success of contextual inference is measured by how well the inferred context vector $\vc$ can be used to recall the whole episode, shown in the bottom panel of Figure~\ref{fig:context_inference_heatmap}. We compare to an ideal observer which commits to one of the consistent context vectors until evidence refutes it (i.e., the upper bound), and find that our system tracks that closely. This experiment demonstrates the value of a system that can perform inference (i.e., ``settle''): we can run the memory \emph{backwards} to infer a hypothesis, and then run the memory forward to test that hypothesis. Standard feedforward systems typically need multiple layers for this kind of inference \citep{olsson2022,bietti2023,sanford2024}, which we achieve with settling ``depth''.

\subsection{Key-value memory}
\label{sec:key_value_memory}
The task is to store a sequence of keys paired with values, and retrieve the appropriate value when cued with a key. We compare to a modern DeltaNet implementation \citep{yang2024c} on the MQAR dataset \citep{arora2024}, which tests this key-value memory capability, particularly in long-range retrieval scenarios.

\begin{table}[ht]
\centering
\small
\setlength{\tabcolsep}{4.5pt}
\caption{Comparison of DeltaNet \citep{yang2024c} and DeltaNet-PC at width $D=256$ on MQAR (512-token streams, 8,192 vocabulary). $^{a}$Distinct training streams needed to achieve perfect test recall (one seed per model). $^{b}$Fast-weight floats: DeltaNet's per-head states, DeltaNet-PC's memory $\mS$ and RLS inverse $\mP$. $^{c}$Excludes token embedding parameters. Memory measurements are based on a batch of 32 on an 80 GB GPU. Test accuracy is over 3 seeds for DeltaNet-PC and 1 seed for DeltaNet.}
\begin{tabular}{@{}lcccccccccc@{}}
\toprule
 & \multicolumn{2}{c}{MQAR} & \multicolumn{2}{c}{Local in} & & & \multicolumn{3}{c}{Peak training memory at $L$} \\
\cmidrule(lr){2-3} \cmidrule(lr){4-5} \cmidrule(lr){8-10}
Model & Test acc. & Streams$^{a}$ & time & space & Params & State$^{b}$ & 512 & 65k & 131k \\
\midrule
DeltaNet    & 1.000 & 1{,}024   & \textcolor{red}{$\times$}   & \textcolor{red}{$\times$}   & 0.53\,M$^{c}$ & 65\,k  & 0.7 GB & 55 GB & OOM \\
DeltaNet-PC (ours) & 1.000 & $32$ & \textcolor{green}{\checkmark} & \textcolor{green}{\checkmark} & 0.20\,M       & 131\,k & 0.25 GB & 0.25 GB & 0.25 GB \\
\bottomrule
\end{tabular}
\label{tab:deltanet_comparison}
\end{table}

Our system achieves 100\% test set accuracy in as few as 32 training examples (Table~\ref{tab:deltanet_comparison}), whereas DeltaNet needs 1,024 (Appendix~\ref{app:eval_details} reports untrained and RoPE-on controls). 
We also note that by avoiding BPTT, our system's training memory does not grow with the sequence length $L$, whereas DeltaNet's training memory grows linearly with $L$. MQAR, with $V=8{,}192$, also shows that our system handles a large vocabulary without adding parameters.

We also evaluate our system's performance on image data, where inputs can be correlated, and the set of classes is open-ended, in contrast with a fixed token vocabulary. For this we use Omniglot characters \citep{lake2015} as keys and MNIST digits \citep{lecun1998} as values (see Appendix Figure~\ref{fig:omniglot_mnist_example} for a data example). Importantly, the Omniglot character classes in the test set are completely disjoint from those in the training set, testing generalization.

On the Omniglot-MNIST key-value memory task with 128 key-value pairs in 256 dimensions (i.e., $L=256$, a full memory load) our system achieves 99.99\% test set
accuracy (3 seeds) showing that it supports image data and can generalize to unseen character classes. We also note that our system learns to retrieve information in arbitrary order, meaning our system can learn to retrieve values by position (see single episode and multi-episode sequential memory tasks) or by content, depending on the task need. In linear attention models, addressing (i.e., keys and queries) is learned with BPTT. Our system learns addressing from local signals: training speeds up the rotary frequencies when position matters and slows them when it does not (Appendix~\ref{app:rope}).

\section{Limitations}
\label{sec:limitations}
The primary limitation of any predictive coding based system is the need to settle latents, which is typically performed with iterative optimization.\footnote{On neuromorphic analog hardware, physical dynamics (i.e., Kirchhoff's and Ohm's laws) perform settling in parallel instantaneously \citep{kendall2020}, but on digital hardware, messages must be passed between neurons.} We document our system's compute costs in Appendix~\ref{app:compute_cost}. While we have demonstrated the benefits that iterative settling can provide (Section~\ref{sec:multiepisode}), future work must continue to reduce the number of iterations required to perform inference, and amortized inference \citep{tschantz2023} provides a promising direction. We also mitigate this limitation by using the RLS update rule (Equation~\ref{eq:rls_update}) to update the memory, which is equivalent to settling $\Delta\mS_t$ under the energy function in Equation~\ref{eq:energy_t}, but is more efficient since it doesn't require iterative optimization.

\section{Discussion}
\label{sec:discussion}
In this work, we introduced DeltaNet-PC, a predictive coding system that learns to perform test-time verbatim memorization of sequential information and generalizes to content unseen in training. Verbatim memorization is a precursor capability for systems that learn \emph{what} to store, which we leave to future work. We demonstrated how our method tackles several common challenges for memory systems, including non-Markovian dynamics, overlapping episodes, and key-value retrieval. We also showed that our system can learn to infer context from observed transitions, which is enabled by iterative settling and compute ``depth''. We enforced strict requirements for local learning in space and time and bounded memory in the sequence length, which positions our system for online learning scenarios and in-memory hardware implementations. 

Treating matrix-valued updates as latent variables, as opposed to slowly learned parameters, allows us to show that predictive coding natively provides linear attention machinery. We compared our work quantitatively to other common memory rules with and without learning and a modern DeltaNet implementation. We also showed mathematically and qualitatively how our work relates to other pivotal works in the linear attention and associative memory literature.

\newpage
\subsection*{AI use statement}
In this work, we (the authors) wrote the first version of the codebase, which provided a clear scaffold for how to run experiments, how to modify the energy function, etc. We then used AI tools to assist with writing code for additional experiments and creating figures and tables. AI was also useful for identifying related works. AI played a role in brainstorming solutions to challenges encountered throughout this project. For instance, we (the authors) intuitively expected that our energy function would require a term to account for interference with past keys. We conceptualized this term as $\sum_{i=1}^{t-1} \lVert \Delta\mS_t \vk_i \rVert^2$, which would naively require storing all past keys, growing our memory footprint with the sequence length $L$. We used AI to help us brainstorm a fixed-size memory solution, which ultimately led to the recursive least squares (RLS) update rule in Equation~\ref{eq:rls_update}. Notably, we did not use AI tools to \emph{write} the main text of this paper. All em dashes, typos, etc. are the work of the authors. Some sections of the appendix were initially drafted with AI assistance but all were subsequently revised by the authors.

\subsection*{Ethics statement}
We do not foresee any ethical concerns with our work at this stage. Our work is a proof-of-concept for a predictive coding system that learns to perform test-time verbatim memorization of sequential information.

\subsection*{Reproducibility statement}
The main text describes the model, the training procedure and the experiments, and the appendices provide all remaining implementation details. Appendix~\ref{app:full_energy} provides the complete energy function, Appendix~\ref{app:datasets} describes the datasets, Appendix~\ref{app:training} provides implementation details, Appendix~\ref{app:hparams} provides hyperparameters, Appendix~\ref{app:ablations} provides ablations, and Appendix~\ref{app:compute_cost} provides compute cost. We will release our codebase upon publication.

\ificlrfinal
\subsubsection*{Acknowledgments}
We would like to thank Colin Bredenberg, Blake Richards, Charlotte Volk, and Jeffrey Seely for their helpful discussions and feedback on the manuscript. This work was supported by funds provided by the National Science Foundation and by DoD OUSD (R\&E) under Cooperative Agreement PHY-2229929 (The NSF AI Institute for Artificial and Natural Intelligence), and the US National Institutes of Health Grant S10OD028632-01.
\fi

\newpage
\bibliographystyle{iclr2027_conference}
\bibliography{bib}

\newpage
\appendix
\section{Extended Related Work}
\label{app:related}

\textbf{Predictive coding.} Predictive coding (PC) \citep{rao1999,friston2005} casts cortical learning as energy descent on layerwise prediction errors, resulting in a local Hebbian-style learning rule \citep{bogacz2017}. Recent work has improved PC at scale via scaling-aware weight parameterizations \citep{innocenti2025}, JAX-based inference \citep{innocenti2024}, and benchmarking \citep{pinchetti2025,salvatori2024}. We build on this core machinery and add a memory module, where matrix-valued updates are themselves PC latents, as opposed to slowly learned weights.

\textbf{Sequential predictive coding.} Several models extend PC to sequences. \citet{tang2023} introduce \emph{temporal predictive coding} (tPC), a recurrent PC model in which a hidden state predicts the next observation. \citet{millidge2024} link PC to Kalman filtering for temporal prediction. These systems treat the recurrent weight as a slowly learned parameter, so they cannot learn new key-value information at test time. Once training ends the model can recall only what it was trained on. Our model adds a fast-weight memory whose update emerges from the energy function used for inference.

\textbf{Predictive coding for associative memory.} A complementary line of work connects PC to associative memory in the static setting: covariance-based hetero-associative memory \citep{tang2023a}, auto-associative variants \citep{salvatori2021}, the continually-learnable BayesPCN \citep{yoo2022}, and a unification of PC with modern Hopfield networks \citep{millidge2022c,ramsauer2021,krotov2016,chaudhry2023} formalized as kernel memory networks \citep{iatropoulos2022}. These systems operate on static patterns. We instead tie a hetero-associative memory to a sequential model, with each step binding a key (the previous latent) to a value (the current latent) inside a fixed-size matrix.

\textbf{Linear attention and test-time learning.} Linear-attention transformers re-express attention as a recurrent state with a fixed-size matrix-valued memory \citep{katharopoulos2020,choromanski2021,peng2021,sun2023,yang2024}. \citet{schlag2021} establish that linear transformers are equivalent to fast-weight programmers, with the canonical update being the DeltaNet / Widrow-Hoff rule \citep{widrow1960}. Earlier fast-weight programmer work includes \citet{hinton1987,ba2016,irie2021}. In these works the update rule is a design decision: \citet{schlag2021} weigh the delta rule against purely additive and gated alternatives, select it for how well it manages memory, and learn a per-step gate $\beta_t$ to set the write strength. The surrounding architecture is then trained by backpropagation through time, as is the case with more modern DeltaNet variants \citep{yang2024c}. More recent work solves the regularized least-squares problem exactly at every step, with recursive least squares (RLS) \citep{vonoswald2024} or with conjugate gradients \citep{vonoswald2026}. In our system many of the design decisions of prior works emerge naturally from our energy function: the outer-product form, the erase term, the memory write strength, and the use of RLS. Importantly, our memory model is embedded within a system that is trained without backpropagation through time (BPTT).

\textbf{Biologically plausible learning beyond BPTT.} Our work draws inspiration from the enviable properties of neuroscience and biology, which learn without backpropagation through time \citep{lillicrap2020}, support online updates, operate within a fixed memory footprint, and have remarkable energy efficiency. Our goal is to build a system with these properties, not necessarily pursue bio-realism for its own sake. Prior works that have explored biologically plausible learning include target propagation \citep{lee2015,aceituno2025}, equilibrium propagation \citep{scellier2017,scurria2026a}, generalized recirculation \citep{oreilly1996}, nonlinear backpropagation \citep{hertz1997}, synthetic gradients \citep{jaderberg2017}, and predictive E-prop \citep{noe2026}.

\textbf{Complementary learning systems.} Our split between fast online memory ($\mS$) and slow synaptic weights mirrors the hippocampus/neocortex split of \citet{mcclelland1995,kumaran2016}.

\section{Energy Function Derivations}
\label{app:full_energy}
\subsection{Partial energy function}
We now prove Lemma~\ref{lem:taylor_expansion} stated in the main text.
\addtocounter{lemma}{-1}
\begin{lemma}
For any $\Delta$
\begin{equation*}
    E_t(\mS_{t-1} + \Delta) = E_{t-1}(\mS_{t-1}) + \langle\nabla E_{t-1}(\mS_{t-1}), \Delta \rangle + \tfrac12 \Tr(\Delta(\epsilon\mathbf{I} + \mC_{t-1})\Delta^{\top}) +
    \tfrac12\lVert \vr_t - \Delta\vk_t\rVert^2.
\end{equation*}
\end{lemma}
\begin{proof}

Recall that $E_t(\mS_{t-1} + \Delta) = E_{t-1}(\mS_{t-1} + \Delta) + \tfrac12\lVert \vv_t - (\mS_{t-1} + \Delta)\vk_t\rVert^2$.

A second-order Taylor expansion of $E_{t-1}(\mS_{t-1} + \Delta)$  about $\mS_{t-1}$, which is exact because $E_{t-1}$ is quadratic, gives
\begin{align*}
E_{t-1}(\mS_{t-1} + \Delta) = E_{t-1}(\mS_{t-1}) + \langle\nabla E_{t-1}(\mS_{t-1}), \Delta\rangle + \frac12\Tr(\Delta \mH_{t-1}\Delta^{\top}),
\end{align*}
where $\mH_{t-1} = \nabla^2 E_{t-1}$ is the Hessian of $E_{t-1}$, which does not depend on the expansion point.

Differentiating $E_{t-1}(\mS)$ once gives
\begin{align*}
\nabla E_{t-1}(\mS) &= \nabla\left( \frac{\epsilon}{2}\lVert \mS \rVert_F^2 + \sum_{i<t} \tfrac12 \lVert \vv_i - \mS \vk_i \rVert^2 \right) \\
&= \epsilon \mS - \sum_{i<t} (\vv_i - \mS \vk_i)\vk_i^{\top}.
\end{align*}

Differentiating $E_{t-1}(\mS)$ a second time gives
\begin{align*}
\nabla^2 E_{t-1}(\mS) &= \nabla\left( \epsilon \mS - \sum_{i<t} (\vv_i - \mS \vk_i)\vk_i^{\top} \right) \\
&= \epsilon \mathbf{I} + \sum_{i<t} \vk_i \vk_i^{\top}\\
&= \epsilon \mathbf{I} + \mC_{t-1}.
\end{align*}
This completes the proof for the first three terms of the lemma.

The last term follows from our definition of $\vr_t = \vv_t - \mS_{t-1}\vk_t$, which means that $\vv_t - (\mS_{t-1} + \Delta)\vk_t = \vr_t - \Delta\vk_t$.
\end{proof}

Since $\mS_{t-1}$ is a ``fixed'' quantity with respect to $E_t$, the first term in Lemma~\ref{lem:taylor_expansion} will be a constant when optimizing $E_t$. Furthermore, $\mS_{t-1}$ minimized $E_{t-1}(\mS)$ (and this holds inductively down to the base case of $\mS_0 = 0 = \argmin_{\mS} E_0(\mS)$), which means that $\nabla E_{t-1}(\mS_{t-1}) = 0$. This allows us to drop the second term in Lemma~\ref{lem:taylor_expansion}. We now have a closed-form energy function for $E_t$:
\begin{equation}
    E_t(\mS_{t-1} + \Delta) = \underbrace{\frac12\lVert \vr_t - \Delta\vk_t\rVert^2}_{\text{new pair}} + \underbrace{\frac12\Tr(\Delta\mC_{t-1}\Delta^{\top})}_{\text{stored pairs}} + \underbrace{\frac{\epsilon}{2}\lVert \Delta \rVert_F^2}_{\text{ridge}} + \text{ const.}
\end{equation}

\paragraph{Algebra for $\Delta\mS_t$ update.} Below we show algebra for Equation~\ref{eq:delta_s_t_update}. We take the gradient of $E_t(\mS_{t-1} + \Delta)$ with respect to $\Delta$ and set it to zero to find the stationary point:
\begin{align}
    \nabla_{\Delta} E_t(\mS_{t-1} + \Delta) &= -(\vr_t - \Delta\vk_t)\vk_t^{\top} + \Delta\mC_{t-1} + \epsilon \Delta = 0, \nonumber \\
    \Delta(\vk_t\vk_t^{\top} + \mC_{t-1} + \epsilon \mathbf{I}) &= \vr_t\vk_t^{\top}, \nonumber \\
    \Delta &= \vr_t\vk_t^{\top}(\vk_t\vk_t^{\top} + \mC_{t-1} + \epsilon \mathbf{I})^{-1}, \label{eq:delta_s_t}
\end{align}
where we have used the fact that $\vk_t\vk_t^{\top} + \mC_{t-1} + \epsilon \mathbf{I}$ is a positive definite matrix and therefore invertible, which also means that due to $E_t$'s convexity in $\Delta$, the stationary point is the global minimum. Let us define $\mA_{t-1} = \mC_{t-1} + \epsilon \mathbf{I}$ and we now have the following update rule for the latent variable $\Delta\mS_t$:
\begin{equation}
\Delta\mS_t = \vr_t\vk_t^{\top}(\vk_t\vk_t^{\top} + \mA_{t-1})^{-1}
\end{equation}

\paragraph{Online updates.} Recall that $\mP_{t-1} = \mA_{t-1}^{-1}$. The Sherman-Morrison formula states that for an invertible matrix $\mA$ with inverse $\mP$ we have
\begin{equation}
    (\mA + \vk\vk^{\top})^{-1} = \mP - \frac{\mP\vk\vk^{\top}\mP}{1 + \vk^{\top}\mP\vk}.
\end{equation}

Let $s_t := \vk_t^{\top}\mP_{t-1}\vk_t$ and $g_t := \frac{1}{1+s_t}$. Plugging the Sherman-Morrison formula into Equation~\ref{eq:delta_s_t_update} gives the following update rule for $\Delta\mS_t$:
\begin{align*}
\Delta\mS_t &= \vr_t\vk_t^{\top}\left(\mP_{t-1} - \frac{\mP_{t-1}\vk_t\vk_t^{\top}\mP_{t-1}}{1 + \vk_t^{\top}\mP_{t-1}\vk_t}\right) \\
&= \vr_t\left(\vk_t^{\top}\mP_{t-1} - \frac{\vk_t^{\top}\mP_{t-1}\vk_t\vk_t^{\top}\mP_{t-1}}{1 + \vk_t^{\top}\mP_{t-1}\vk_t}\right) \\
&= \vr_t\left(\vk_t^{\top}\mP_{t-1} - \frac{s_t\vk_t^{\top}\mP_{t-1}}{1 + s_t}\right) \\
&= \vr_t\left(\vk_t^{\top}\mP_{t-1}(1 - \frac{s_t}{1 + s_t})\right) \\
&= \vr_t\left(\frac{\vk_t^{\top}\mP_{t-1}}{1 + s_t}\right) \\
&= g_t \vr_t \left(\mP_{t-1} \vk_t\right)^{\top},
\end{align*}
where the last line relies on the symmetry of $\mP_{t-1}$. This \emph{is} the recursive least squares (RLS) update rule \citep{plackett1950}.

Importantly, this update rule no longer requires an inverse to be computed at every time step, which makes it feasible for online learning scenarios. It only requires the storage of $\mP_{t-1}$, which is a $D \times D$ matrix, and complies with our fixed-size memory footprint constraint. Furthermore, recall that $\mS_t$ is also a $D \times D$ matrix, so the total memory footprint of our memory module is $2D^2$, which is fixed and does not grow with the sequence length $L$. At each timestep, we update $\mS_t$ and $\mP_t$ as
\begin{gather}
    \mS_t = \mS_{t-1} + g_t \vr_t \left(\mP_{t-1} \vk_t\right)^{\top}, \qquad \mS_0 = \mathbf{0} \\
    \mP_t = \mP_{t-1} - g_t (\mP_{t-1}\vk_t)(\mP_{t-1}\vk_t)^{\top}, \qquad \mP_0 = \mathbf{I}/\epsilon.
\end{gather}

\subsection{Relating the energy function to linear attention}
\label{app:linear_attention}
Equations~\ref{eq:energy_t} and \ref{eq:rls_update} can be interpreted relative to common choices that have been made in the linear attention literature, in particular the (i) outer-product form found in \citep{katharopoulos2020}, (ii) the erase term found in \citep{schlag2021}, (iii) the write strength (e.g., the gate $\beta_t$ that \citet{schlag2021} must learn), and (iv) the least-squares objective over all past pairs, which the Mesa layer solves with RLS \citep{vonoswald2024}, all of which follow from our energy function.

\paragraph{(i) Outer-product form.} If we write $\vw_t := g_t\,\mP_{t-1}\vk_t$, then Equation~\ref{eq:rls_update} becomes
\begin{equation}
    \mS_t = \mS_{t-1} + \vr_t \vw_t^{\top}, \qquad \mS_0 = \mathbf{0}
\end{equation}
consistent with the outer-product form chosen for linear attention \citep{katharopoulos2020}.

\paragraph{(ii) Erase term.} Moving from vanilla linear attention $\mS_t = \mS_{t-1} + \vv_t\vk_t^{\top}$ to DeltaNet, \citet{schlag2021} introduced an erase term to reduce interference, given by $\mS_t = \mS_{t-1} + \beta_t\,\vv_t\vk_t^{\top} - \beta_t\,\mS_{t-1}\vk_t\vk_t^{\top}$. Expanding $\vr_t$ in Equation~\ref{eq:rls_update} and writing $\vw_t := \mP_t\vk_t$ gives
\begin{equation}
    \mS_t = \mS_{t-1} + \underbrace{\vv_t\vw_t^{\top}}_{\text{write}} - \underbrace{\mS_{t-1}\vk_t\vw_t^{\top}}_{\text{erase}}
\end{equation}
where our system has the same write and same erase, with $\vw_t$ in place of $\beta_t\vk_t$, where $\vw_t$ is set by the energy.

\paragraph{(iii) Write strength.} DeltaNet's write strength gate $\beta_t$ is a learned scalar applied to every direction. Ours is a matrix $\mP_t = (\epsilon\mathbf{I} + \mC_t)^{-1}$. It results from an interaction with past keys (recall $\mC_t = \sum_{i=1}^{t} \vk_i \vk_i^{\top}$) and the current key $\vk_t$. Where DeltaNet has $\beta_t \vk_t$, we have $\mP_t \vk_t$. The Mesa layer of \citet{vonoswald2024} uses the same matrix, with a learned ridge in place of our fixed $\epsilon$.

\paragraph{(iv) Memory interference term.} The term $\frac12\Tr(\Delta\mC_{t-1}\Delta^{\top})$ from our energy Equation~\ref{eq:energy_t} can be rewritten as: $\frac12\sum_{i<t} \lVert \Delta\vk_i\rVert^2$. It has an elegant interpretation: penalize any write $\Delta$ in proportion to how much it changes the retrieval on any past key $\vk_i$. Any nonzero $\Delta\vk_i$ contributes to interference so the energy function is actively trying to minimize it. \citet{vonoswald2024} define the Mesa layer as the exact minimizer of the same ridge-regression objective over all past pairs (Equation~\ref{eq:memory_energy}) and compute it with RLS. Their motivation is in-context learning: the layer computes the least-squares solution in one step that trained Transformers approximate across many layers. The interference term is implicit in their objective. Our derivation makes it explicit and arrives at it from the PC energy.

\subsection{Value-orthogonality regularizer}
\label{app:value_ortho}
In order to further improve key geometry, we can also engage an optional regularizer that encourages values $\vv_t$ to be orthogonal to prior values $\vv_i$ for $i < t$. Recall that keys are readouts of values---$\vv_t \approx \enc(\vx_t)$ and $\vk_{t+1} \approx R(\mW_k \vv_t, t+1)$---so keys inherit any improved geometry from the values. We remove magnitude information so that we only encourage angular separation of values. Let $\hat{\vv} = \vv / \lVert \vv \rVert$ be a normalized value vector, $\mC_{t-1}^v = \sum_{i < t}\hat{\vv}_i\hat{\vv}_i^{\!\top}$ be the accumulated outer-product of prior normalized values, and $\tilde{\mC}_{t-1}^v = \mC_{t-1}^v / \max(\Tr(\mC_{t-1}^v), 1)$ be the normalized version of $\mC_{t-1}^v$ that is invariant to the number of prior values. The vanilla form of the regularizer is
\begin{equation}
\mathcal{R}^{\vv}_t \;=\; \tfrac{1}{2} \hat{\vv}_t^{\top} \tilde{\mC}_{t-1}^v \hat{\vv}_t
\;=\; \frac{1}{2\max(t-1,\,1)}\sum_{i < t}(\hat{\vv}_t^{\!\top}\hat{\vv}_i)^2.
\label{eq:reg_v}
\end{equation}
Note that in some cases, one may observe the same value multiple times in a sequence, and in that case, we may not want to try to orthogonalize the value with respect to itself. In that case, we can use a modified version of the regularizer that ignores the current value $\vv_t$ in the accumulated outer-product. Let $\hat{\ve}_t = \enc(\vx_t) / \lVert \enc(\vx_t) \rVert$ (stop-gradiented) and $m = \hat{\ve}_t^{\top}\tilde{\mC}_{t-1}^v\hat{\ve}_t$ be the magnitude of the current value in the accumulated outer-product. The modified regularizer is
\begin{gather}
\mathcal{R}^{\vv}_t \;=\; \tfrac{1}{2} \hat{\vv}_t^{\top} \tilde{\mC}_{t-1}^v \hat{\vv}_t - \tfrac{1}{2} m (\hat{\ve}_t^{\top}\hat{\vv}_t)^2 \;=\; \tfrac{1}{2}\hat{\vv}_t^{\top} \mM_t \hat{\vv}_t,\\
\mM_t \;=\; \tilde{\mC}_{t-1}^v - m \hat{\ve}_t\hat{\ve}_t^{\!\top}.
\label{eq:reg_v_mod}
\end{gather}
Notably, the memory footprint of this regularizer is $O(D^2)$, which is independent of the sequence length $L$.

\subsection{Full energy function}
\label{app:full_energy_fn}

Every factor of Eq.~\ref{eq:full_joint} contributes one quadratic term of the form $\tfrac12\lVert \text{variable} - \text{mean} \rVert^2$, the stored pairs term comes from Lemma~\ref{lem:taylor_expansion}, and the value-orthogonality regularizer is detailed in \S\ref{app:value_ortho}. The energy at position $t$ is a function of three latent variables $(\vv_t, \vk_t, \Delta\mS_t)$ and is given by
\begin{equation}
\begin{aligned}
E_t(\vv_t, \vk_t, \Delta\mS_t)
\;=\;&
\underbrace{\tfrac12\lVert \vx_t - \dec(\vv_t)\rVert^2}_{\text{decoder}}
\;+\;
\underbrace{\tfrac{\lambda_{\text{ev}}}{2}\lVert \vv_t - \enc(\vx_t)\rVert^2}_{\text{encoder}}
\;+\;
\underbrace{\tfrac12\lVert \vv_t - (\mS_{t-1} + \Delta\mS_t)\,\vk_t\rVert^2}_{\text{memory energy (new pair)}} \\
&+\;
\underbrace{\tfrac{\lambda_k}{2}\lVert \vk_t - R(\mW_k\,\enc(\vx_{t-1}), t)\rVert^2}_{\text{address prior}}
\;+\;
\underbrace{\tfrac12\Tr(\Delta\mS_t\,\mC_{t-1}\,\Delta\mS_t^{\top})}_{\text{stored pairs}}\\
&\;+\;
\underbrace{\tfrac{\epsilon}{2}\lVert \Delta\mS_t\rVert_F^2}_{\text{ridge}}
\;+\;
\lambda_v\,\mathcal{R}^{\vv}_t(\vv_t)
\;+\; \text{const},
\end{aligned}
\label{eq:full_energy}
\end{equation}
with $\vx_{t-1}$ and $\vx_t$ observed, and $\mS_{t-1}$, $\mC_{t-1}$ and $\mC^v_{t-1}$ the memory as of the previous step, $t-1$. Figure~\ref{fig:architecture} visualizes Eq.~\ref{eq:full_energy}: one Gaussian term per arrow, plus the memory terms. A deeper encoder or decoder would add one term of the form $\tfrac12\lVert \vz_{\ell} - \mW_\ell\, g(\vz_{\ell-1})\rVert^2$ per layer (Eq.~\ref{eq:pc_energy}).

\paragraph{Write and learn passes.} The training procedure (Algorithm~\ref{alg:training}) minimizes Eq.~\ref{eq:full_energy} over different variables depending on the type of event observed in the stream: an observation $(\vx_{t-1}, \vx_t, t)$ (the write) or a query $(\vx^{\text{in}}, \vy, p)$ (the learn pass).

\emph{The write} (an observation $(\vx_{t-1}, \vx_t, t)$) minimizes the energy function over $\Delta\mS_t$ alone, with the value and the address at their prior means, $\vv_t = \enc(\vx_t)$ and $\vk_t = R(\mW_k\,\enc(\vx_{t-1}), t)$. Only the three $\Delta\mS_t$ terms depend on $\Delta\mS_t$, so the minimizer is the closed form of Eq.~\ref{eq:delta_s_t_update}, computed recursively as Eq.~\ref{eq:rls_update}. It commits updates to $\mS_t$, $\mP_t$, and $\tilde{\mC}^v_t$.

\emph{The learn pass} (a query $(\vx^{\text{in}}, \vy, p)$, Table~\ref{tab:query_schedules}) minimizes the same energy over the value and the address key, $(\vv, \vk)$, with the memory held fixed and the observation $\vy$ clamped at the output. In particular, that is the energy function restricted to the following:
\begin{equation}
\begin{aligned}
E^{\text{learn}}_t(\vv, \vk;\, \vx^{\text{in}}, \vy, p)
\;=\;&
\tfrac12\lVert \vy - \dec(\vv)\rVert^2
\;+\;
\tfrac{\lambda_{\text{ev}}}{2}\lVert \vv - \enc(\vy)\rVert^2
\;+\;
\tfrac12\lVert \vv - \mS_{t}\vk\rVert^2 \\
&+\;
\tfrac{\lambda_k}{2}\lVert \vk - R(\mW_k\,\enc(\vx^{\text{in}}), p)\rVert^2
\;+\;
\lambda_v\,\mathcal{R}^{\vv}_t(\vv).
\end{aligned}
\label{eq:learn_energy}
\end{equation}
The pair is optimized for $T$ gradient steps, initialized from the same prior means the write used, and the parameter step is the gradient of Eq.~\ref{eq:learn_energy} at the settled point with respect to the parameters: $\dec$, $\enc$, $\mW_k$ and the rotary frequencies from the address prior.

\paragraph{Recall.} At test time, the default is to read out $\mS_t\vk_t$ at the feedforward address and decode to the observation space. The decoded token is re-encoded to form the next address. Nothing is written (Algorithm~\ref{alg:training}). Optionally, a write-time settle can be run to refine the value and memory update (with the key held fixed at the feedforward address).

\subsection{Algorithmic description of DeltaNet-PC}
\label{app:algorithm}
Algorithm~\ref{alg:training} gives the full training step described in \S\ref{sec:training}.

\begin{algorithm}[h]
\DontPrintSemicolon
\SetKwComment{Comment}{$\triangleright$\ }{}
\caption{DeltaNet-PC on one training example / event stream containing observations and queries (Table~\ref{tab:query_schedules}).}
\label{alg:training}
\KwIn{Event stream $e_1, \dots, e_N$, each event an observation $(\vx_{t-1}, \vx_t, t)$ or a query $(\vx^{\text{in}}, \vy, p)$: a cue, a target, and the position; parameters $\boldsymbol{\theta} = (\boldsymbol{\theta}_{\enc}, \boldsymbol{\theta}_{\dec}, \mW_k, \boldsymbol{\theta}_R)$.}
$\mS \gets \mathbf{0}$;\quad $\mP \gets \mI/\epsilon$;\quad $\mC^v \gets \mathbf{0}$;\quad $\bar{\vg} \gets \mathbf{0}$;\quad $n \gets 0$\;
\ForEach{event $e$ in the stream}{
  \uIf{$e$ is an observation ${(\vx_{t-1}, \vx_t, t)}$}{
    \Comment{memory update: minimizer of $E_t$ (Eq.~\ref{eq:energy_t})}
    $\vk \gets R(\mW_k\,\enc(\vx_{t-1}),\, t)$;\quad $\vv \gets \enc(\vx_t)$\;
    $\vr \gets \vv - \mS\vk$;\quad $\vw \gets \mP\vk$;\quad $g \gets 1 / (1 + \vk^{\top}\vw)$\;
    $\mS \gets \mS + g\,\vr\,\vw^{\top}$;\quad $\mP \gets \mP - g\,\vw\vw^{\top}$;\quad $\mC^v \gets \mC^v + \hat{\vv}\hat{\vv}^{\top}$\;
  }
  \Else(query ${(\vx^{\text{in}}, \vy, p)}$){
    \Comment{settle key and value with $\boldsymbol{\theta}$ and $\mS$ frozen}
    $\tilde{\vk} \gets R(\mW_k\,\enc(\vx^{\text{in}}),\, p)$;\quad $\tilde{\vv} \gets \enc(\vy)$\;
    \For{$\tau = 1, \dots, T$}{
      $(\tilde{\vk}, \tilde{\vv}) \gets (\tilde{\vk}, \tilde{\vv}) - \eta_x\,\nabla_{(\tilde{\vk}, \tilde{\vv})}\,E^{\text{learn}}(\tilde{\vk}, \tilde{\vv};\, \vx^{\text{in}}, \vy, p)$
    }
    $\bar{\vg} \gets \bar{\vg} + \nabla_{\boldsymbol{\theta}}\,E^{\text{learn}}(\tilde{\vk}, \tilde{\vv};\, \vx^{\text{in}}, \vy, p)$;\quad $n \gets n + 1$ \Comment*[r]{Parameter update}
  }
}
$\boldsymbol{\theta} \gets \boldsymbol{\theta} - \eta_\theta\,\tfrac{256}{n}\,\bar{\vg}$ \Comment*[r]{one step per episode: scaled by $\times\, 256$}
\Return $\boldsymbol{\theta}$\;
\end{algorithm}

\begin{table}[h]
\centering
\footnotesize
\setlength{\tabcolsep}{4pt}
\caption{The event streams for the three tasks.}
\begin{tabular}{@{}p{2.1cm}p{6.0cm}p{5.2cm}@{}}
\toprule
Task & Event stream & Test scoring \\
\midrule
Sequential recall &
Observations and queries alternate: after each observation $(\vx_{t-1}, \vx_t, t)$ the learner poses its own query $(\vx_{t-1}, \vx_t, t)$, to retrieve what was just written at its own position, against the memory as it stands. $L$ observations and $L$ queries per stream; online learning. &
{\raggedright Ground truth $\vx_t$ compared to prediction $\hat{\vx}_t$ decoded from $\mS_L R(\mW_k\,\enc(\hat{\vx}_{t-1}), t)$ for $t = 1, \dots, L$, the final memory frozen.\par} \\
\addlinespace
Official MQAR &
The $128$ context observations (64 pairs), then the query section: at each labelled position $t$ the task's query $(\vx_{t-1}, \vy_t, t)$, the query key and its value, precedes the observation that writes the filler $\vx_t$. Every position is written. &
{\raggedright Ground truth $\vy_t$ compared to prediction $\hat{\vy}_t$ decoded from $\mS_{t-1}R(\mW_k\,\enc(\vx_{t-1}), t)$ for each $t$ containing a query.\par} \\
\addlinespace
Omniglot binding &
Stream of $L$ observations, then the task's queries: all queries have the form $(\vx_{t-1}, \vx_t, L)$ meaning they are queries at the final position $L$. &
{\raggedright Ground truth $\vx_t$ compared to prediction $\hat{\vx}_t$ decoded from $\mS_L\,R(\mW_k\,\enc(\vx_{t-1}), L)$.\par} \\
\bottomrule
\end{tabular}
\label{tab:query_schedules}
\end{table}

\section{Datasets}
\label{app:datasets}
Here we describe each of our datasets in more detail. Table~\ref{tab:datasets} summarizes our four datasets. Each dataset provides a stream of observations $\vx_t \in \R^{d_{\text{obs}}}$, where tokens are represented as $d_{\text{obs}} = 256$-dimensional vectors drawn from a Gaussian codebook that is unit-normalized, and images are represented as $28 \times 28$ pixel arrays flattened to $d_{\text{obs}} = 784$. All tasks use disjoint training, validation, and test splits.

\begin{table}[h]
\centering
\footnotesize
\setlength{\tabcolsep}{4pt}
\caption{Summary of datasets.}
\begin{tabular}{@{}p{2.1cm}p{1.7cm}cp{2.2cm}p{1.7cm}p{1.9cm}p{2.0cm}@{}}
\toprule
Task & Observation & $d_{\text{obs}}$ & $L$, $D$ & $V$ & Train / Val. / Test & Scored units per test set \\
\midrule
Single-episode sequences & token & 256 & $L = 64$ to $512$ at $D = 256$, $L = 512$ at $D = 64$ to $512$, $L = D = 1{,}024$ & 10 tokens & $1{,}024$ / $128$ / $128$ & $128 L$ relay tokens \\
\addlinespace
Multi-episode sequences & token & 256 & $256$ ($K=8 \times \ell=32$), $D = 256$ & 10 tokens & $1{,}024$ / $128$ / $128$ & $32{,}768$ relay tokens \\
\addlinespace
MQAR \citep{arora2024} & token & 256 & $512$, $D = 256$ & $8{,}192$ tokens ($4{,}095$ keys, $4{,}096$ values) & $100{,}000$ (or $32$) / $128$ / $3{,}000$ & $192{,}000$ labelled positions \\
\addlinespace
Omniglot-MNIST binding & image $28 \times 28$ & 784 & $256$ ($128$ pairs), $D = 256$ & 10 digits & $1{,}024$ / $128$ / $128$ & $16{,}384$ probes \\
\bottomrule
\end{tabular}
\label{tab:datasets}
\end{table}

\subsection{Single-episode sequential memory}
\label{app:data_seq}

\paragraph{Generator.} A sequence of $L$ tokens is drawn uniformly from a vocabulary of size $V$. One non-Markovian transition is inserted into each sequence, so that a given token is followed by two different successors---the chorus with two bridges of the song analogy in the main text.

\paragraph{Observations.} Each token is represented as its codebook row in $\R^{256}$, shared by all splits and all seeds. Decoding is performed by finding the minimum distance codebook entry by Euclidean distance.

\paragraph{Sizes.} Each split is composed of independently generated sequences: $1{,}024$ for training and $128$ each for validation and test. The lengths used are $L \in \{64, 128, 256, 512\}$ at $D = 256$, $L = 512$ at $D \in \{64, 128, 256, 512\}$, and $L = D = 1{,}024$, three seeds each. 
$n = 0$  is the untrained control (every learning rate zero), three seeds each.

\paragraph{Metric.} Recall accuracy is the fraction of the $L$ ground truth tokens that are recalled after the memory has been filled. The model starts decoding from a start observation $\vx_0$, queries the memory with the resulting key $\vk_1 = R(\mW_k \enc(\vx_0), 1)$, and decodes the result to form $\hat{\vx}_1$ from $\min_{\vx \in \text{codebook}} \lVert \vx - \hat{\vx} \rVert_2$ where $\hat{\vx}$ is the model's prediction. The model then forms the next key, reads the memory, and decodes to form $\hat{\vx}_2$, and so on until $\hat{\vx}_L$. We call this free-running decode loop the \emph{relay}. Notably, we use no teacher forcing---the model must stay on track through the entire $L$-length sequence to succeed. The test set scores $128 \times L$ tokens, $32{,}768$ at $L = 256$.

\subsection{Multi-episode sequential memory}
\label{app:data_multiep}
\begin{figure}[H]
    \centering
    \includegraphics{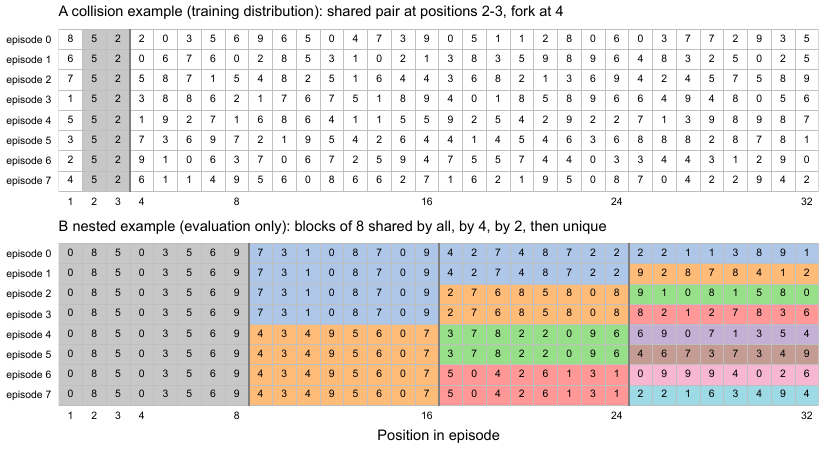}
    \caption{Multi-episode examples. \textbf{A}: A single training example. \textbf{B}: An evaluation example for the contextual inference task from Section~\ref{sec:multiepisode}.}
    \label{fig:multi_episode_examples}
\end{figure}

\paragraph{Generator, collision examples (training and evaluation).} An example is $K = 8$ episodes of $\ell = 32$ tokens, written into a single stream of length $L = 256$. Within-episode tokens are drawn uniformly from a vocabulary of size $V$. All $K$ episodes of an example have the same two tokens at within-episode positions $2$ and $3$. The predecessors at position $1$ are pairwise distinct across the episodes, as are the successors at position $4$. At the end of the shared block the pair (position, token) is identical in every episode, so the successor is ambiguous without context, and a memory keyed on content and position alone reads out the superposition of the $K$ successors. Figure~\ref{fig:multi_episode_examples}A shows an example.

\paragraph{Generator, nested examples (evaluation only).} The contextual inference experiment evaluates DeltaNet-PC on a harder task structure. See Figure~\ref{fig:multi_episode_examples}B for an example. The episode is four blocks of $8$ tokens ($\ell=32$). Block $1$ is shared by all $8$ episodes, block $2$ by $4$, block $3$ by $2$, and block $4$ is unique in each episode. Sibling groups differ at the first token of their block. These examples are generated at evaluation time only.

\paragraph{Episode context vectors.} The dataset provides each episode with a context vector $\vc \in \R^{D}$, which the model consumes as the context tag on the key (Appendix~\ref{app:rope}). Context vectors are unit-modulus phasors---one uniformly random angle per rotation plane. During training the $K$ context vectors are drawn once per run from the run seed and shared by every training example. At validation and test, \emph{every example} draws fresh context vectors, so all reported multi-episode numbers are for never-seen context vectors. Rotary positions for the learned RoPE frequencies restart at every episode start. Recall that in the multi-episode setting, keys are defined as $\vk_t = R(\mW_k\,\enc(\vx_{t-1}), t, \vc)$, where $R$ applies a rotation in each plane of the key vector by the learned RoPE frequency for that plane, scaled by the position $t$ and the context vector $\vc$.

\paragraph{Sizes and metric.} Splits are $1{,}024$ / $128$ / $128$ examples. Recall begins at the start of an episode and proceeds through the episode, as in the single-episode setting. Every position is scored: $128 \times 256 = 32{,}768$ tokens on the test set.

\subsection{MQAR}
\label{app:data_mqar}

\paragraph{Generator.} We use the community-standard multi-query associative recall stream of \citet{arora2024} in its hardest published setting. A stream has $512$ tokens and a vocabulary of $8{,}192$. The first $128$ tokens are $64$ key-value pairs, and the remaining $384$ tokens are the query section: each of the $64$ keys appears exactly once more and every other position holds the filler token $0$. The label at a query position is that key's value, which is the target of the prediction made after reading the query key, while the next input token is the filler. Every other position is unlabelled. The model writes every position of the stream, filler and queries included, and is supervised and scored only at the $64$ labelled positions.

\paragraph{Observations.} Tokens are drawn from rows of an $8{,}192 \times 256$ Gaussian codebook for DeltaNet-PC. The DeltaNet baseline receives its embedding table input. Decoding is performed by finding the nearest of all $8{,}192$ rows by Euclidean distance for DeltaNet-PC and argmax of the logits for the DeltaNet baseline.

\paragraph{Sizes.} We train on Zoology's full $100{,}000$-stream training set, early-stopped on validation set recall accuracy.
The test set is Zoology's $3{,}000$-stream test segment. Early stopping uses our own $128$-stream validation split.

\paragraph{Metric.} Accuracy at the labelled positions of the test set, $3{,}000 \times 64 = 192{,}000$ scored positions, at $D = 256$. The answer at a query is read from the memory before that position is written.

\subsection{Omniglot-MNIST key-value binding}
\label{app:omniglot_mnist}

\begin{figure}[H]
    \centering
    \includegraphics{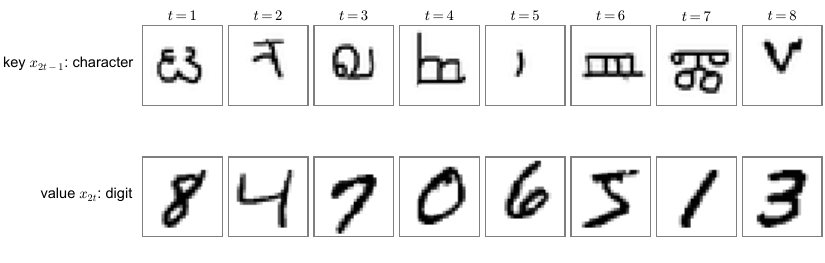}
    \caption{Omniglot-MNIST key-value memory example.}
    \label{fig:omniglot_mnist_example}
\end{figure}

\paragraph{Images.} Keys are Omniglot characters \citep{lake2015} and values are MNIST digits \citep{lecun1998}, both at $28 \times 28$ pixels and flattened to $\R^{784}$. Omniglot images are inverted to bright strokes on a dark background, matching MNIST. Every image is centered by subtracting its own pixel mean, then the mean image of the training pool of its type is subtracted (mean for characters, separate from the mean for digits), and the result is scaled to unit norm. The value vocabulary is one fixed MNIST exemplar per digit---the first occurrence of each digit in the MNIST training set.

\paragraph{Character pools.} Character classes are disjoint by split. Training draws from $964$ background archive classes, and validation and test draw from disjoint halves of the $659$-class evaluation archive ($329$ and $330$ classes), so every character at test time belongs to a class the model never saw in training. Each class has $20$ exemplars, and each character in an example is drawn at random.

\paragraph{Stream layout.} An example consists of $P = 128$ characters from $128$ distinct classes, each once, as interleaved (character, digit) image pairs in a shuffled order, so the stream has $L = 2P = 256$ observations and fills a memory of width $D = 256$. Each character is bound to a digit, where digits are drawn uniformly with replacement, so many characters (keys) share a value. Figure~\ref{fig:omniglot_mnist_example} shows an example.

\paragraph{Sizes and metric.} Splits are $1{,}024$ / $128$ / $128$ examples, and three seeds. After the stream is written, every character (key) of the example is presented again as a query at the final position and the digit read from the finished memory is decoded by the nearest of the ten digit prototypes. The test set scores $128 \times 128 = 16{,}384$ probes.

\section{Architecture, Training, and Experimental Details}
\label{app:training}
In this section, we provide additional details on the architecture, training, and experimental setup used in our work. We refer to our default configuration (Algorithm~\ref{alg:training} with the hyperparameters of Table~\ref{tab:hparams}) as the \emph{recipe}.

\providecommand{\vpsi}{\boldsymbol{\psi}}

\subsection{Rotary position code}
\label{app:rope}

\paragraph{Definition.} Write the $D$ coordinates of a key as $D/2$ planes $(k_i, k_{i + D/2})$. The rotary code $R(\vk, t)$ rotates plane $i$ by the angle $t\,\theta_i$,
\begin{equation}
R(\vk, t)_i \;=\; \begin{pmatrix} \cos t\theta_i & -\sin t\theta_i \\ \sin t\theta_i & \cos t\theta_i \end{pmatrix} \begin{pmatrix} k_i \\ k_{i + D/2} \end{pmatrix},
\qquad
\boldsymbol{\theta}_R = (\theta_1, \dots, \theta_{D/2}) \in [0, \pi]^{D/2}.
\label{eq:rope_def}
\end{equation}
This is rotary positional encoding \citep{su2024} in its non-interleaved form. $R(\cdot, t)$ is orthogonal, so it preserves the key norm, and rotations in the same planes compose by adding angles. Queries are keys, $\vq_t = \vk_t = R(\mW_k\,\enc(\vx_{t-1}), t)$, and the write, the learn-pass read and the recall relay all call the same map, so write and read addresses agree by construction. The official MQAR row uses no rotary code to be consistent with DeltaNet \citep{yang2024c}, so our key is defined as $\mW_k\,\enc(\vx_{t-1})$.

\paragraph{Learned frequencies.} The frequencies are parameters of the model, one scalar per plane in radians per position. They are initialized at the standard table $\theta_i = b^{-2(i-1)/D}$ with $b = 10^4$, so a fresh model addresses exactly like standard RoPE, and they are updated by SGD, after which each $\theta_i$ is projected onto $[0, \pi]$. A band at $0$ ignores position and a band at $\pi$ alternates every position. Nothing in the recipe depends on the sequence length.

\paragraph{How the frequencies learn.} The frequencies appear in exactly one energy term, the address prior of the learn pass, $\tfrac{\lambda_k}{2}\lVert\vk - R(\mW_k\,\enc(\vx^{\text{in}}), p)\rVert^2$ at the query position $p$. They are held constant on every other path, so they only learn from retrieval error. In plane $i$ write the content component as the complex number $a_i$ and the settled address as $k_i$. When the settled address points at the address the same content was written under at position $t_w$, $k_i = e^{\mathrm{i} t_w \theta_i} a_i$, the gradient is
\begin{equation}
\frac{\partial E^{\text{learn}}}{\partial \theta_i} \;=\; -\,\lambda_k\, p\,\lvert a_i \rvert^2 \sin\bigl(\theta_i\,(t_w - p)\bigr).
\label{eq:rope_grad}
\end{equation}
Each band therefore behaves as a phase-locked loop between the query position and the write position, with a force proportional to $p$. Where the task places its queries decides what the frequencies learn (Figure~\ref{fig:rope_spectra}):
\begin{itemize}
\item \textbf{Sequential recall} queries each transition at its own position, so the binding just written exerts no force. What moves the frequencies is interference with other occurrences of the same token stored at other positions. The $76$ bands that start slower than one cycle over the sequence are lifted by a median factor of $170$ to $230$ across seeds, to a spectrum near one radian per position, and $4$ to $6$ bands end static. In brief, the code speeds up to avoid collisions of the same token at other positions, thus improving address orthogonality.
\item \textbf{Multi-episode recall} does the same over the episode: the $105$ bands slower than one cycle per episode are lifted by a median factor of $23$ to $29$, and $13$ to $15$ end static.
\item \textbf{Omniglot binding} reads every pair at the final position, after all writes, so $t_w < p$ for every stored pair and the force lowers the slow bands. Of the $128$ bands, $88$ to $94$ end below $10^{-3}$ radians per position, while the fastest $30$ or so keep values near their initialization. The address stops depending on when a pair was written, which is what a content-addressed probe needs.
\end{itemize}

We emphasize that we were very intentional about designing a system that could \emph{learn} to use positional information when needed and ignore it when not needed. The alternative is to hard-code the RoPE frequencies based on task need, which would make our system less general. In general, we made efforts to remove inductive biases, wherever possible.

\begin{figure}[htbp]
    \centering
    \includegraphics{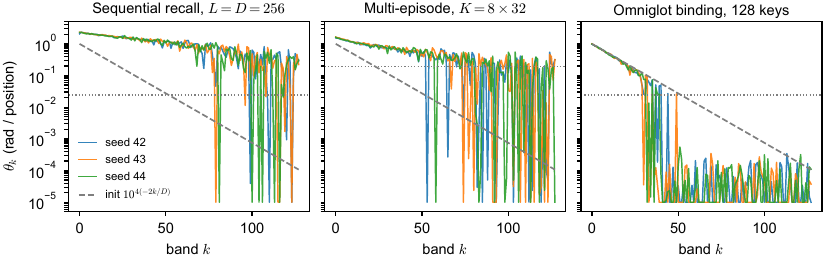}
    \caption{Learned rotary frequencies of the three rows that use them, three seeds each, against the base-$10^4$ initialization (dashed). The dotted line marks one cycle per horizon: $2\pi/256$ on the sequential and Omniglot rows and $2\pi/32$, one episode, on the multi-episode row. Sequential and multi-episode recall raise the slow bands, and Omniglot binding, which reads after every write, sends most of them to zero.}
    \label{fig:rope_spectra}
\end{figure}

\subsection{Context vectors and the multi-episode protocol}
\label{app:context_vectors}

\paragraph{Context vectors.} Every episode $e$ of a multi-episode example carries a context vector $\vc_e \in \R^{D}$ made of $D/2$ unit phasors: draw $\vpsi_e \sim \mathcal{U}(-\pi, \pi)^{D/2}$ and set $\vc_e = (\cos \vpsi_e, \sin \vpsi_e)$. During training the $K$ context vectors are drawn once per run from the training seed and shared by every training example. At validation and test every example draws fresh context vectors, so every reported multi-episode number is recall under contexts never seen in training.

\paragraph{How a context vector enters the key.} The context vector is present at every position of its episode. It does not pass through the encoder or the key map. It multiplies each plane of the key as a complex number, $(k_i + \mathrm{i}\,k_{i + D/2})(\cos \psi_{e,i} + \mathrm{i} \sin \psi_{e,i})$, which is a rotation of plane $i$ by $\psi_{e,i}$, the same planes RoPE rotates. With positions counted within the episode, the address of the transition at within-episode position $p$ of episode $e$ is
\begin{equation}
\vk \;=\; R\bigl(\mW_k\,\enc(\vx_{t-1}),\; p\,\boldsymbol{\theta}_R + \vpsi_e\bigr),
\label{eq:context_vector_compose}
\end{equation}
so position and context compose by adding phases. We wrote this as $R(\cdot, t, \vc)$ in \S\ref{sec:arch}. Both factors are orthogonal and commute, so the tag changes neither the key norm nor anything else in the recipe, and a single-episode run is the case $\vpsi = \mathbf{0}$. A read under the same context vector as the write cancels the tag exactly in the readout, $R(\vk, \vpsi_e)^{\top} R(\vk', \vpsi_e) = \vk^{\top}\vk'$. A read under a different context vector sees a random rotation in every plane, $\sum_i \lvert k_i \rvert^2 \cos(\psi_{e,i} - \psi_{e',i})$, which averages to zero, so keys from different contexts are nearly orthogonal. Multiplying a vector componentwise by random unit phasors is the binding operation of Fourier holographic reduced representations \citep{plate1995,kanerva2009}. What is new here is its use as the context channel of a least-squares fast-weight memory with contexts never seen in training, and the inference of a withheld context by settling (\S\ref{app:ctx_inference}).

\paragraph{Protocol.} The $K$ episodes of an example are written as one stream of length $K\ell$ into one memory, which is not reset inside the example. Rotary positions restart at every episode start, and the first key of every episode is formed from one fixed start observation, the same for all episodes, examples and splits, so the first keys of different episodes differ only through their context vectors. Recall provides each episode with its context vector and runs the decode relay against the finished memory, exactly as in the single-episode protocol.

\subsection{Contextual inference by settling}
\label{app:ctx_inference}

\paragraph{Examples.} An example is $K = 8$ episodes of $\ell = 32$ tokens written into one memory under $8$ context vectors. In the \emph{collision examples}, the training distribution, every episode has its own first token, all eight share the tokens at positions $2$ and $3$, and each has its own token at position $4$, so the shared block is a collision only the context vector resolves. In the \emph{nested examples} each episode is four blocks of eight tokens arranged as a binary tree: all eight episodes agree on block one, groups of four on block two, pairs on block three, and every episode has its own block four (Figure~\ref{fig:multi_episode_examples}). Nested examples are only used at evaluation, on the models trained on collision examples.

\paragraph{Settling (i.e., inference).} The memory is filled under the true context vectors and then held fixed. For each query episode, the context vector is withheld, and a context \emph{latent} $\vc \in \R^{D}$ is inferred from the episode's first $n$ transitions, revealed one at a time. For each revealed transition $(\vx_{p-1}, \vx_p)$ the value is pinned at the encoder's output, $\vv_p = \enc(\vx_p)$, which is what the write stored, and the address is taken to equal its prediction under the current context latent, the limit of infinite address-prior precision. The context latent then settles on the read error of the value,
\begin{equation}
E_p(\vc) \;=\; \tfrac12 \bigl\lVert \enc(\vx_p) - \mS_L\, R\bigl(\mW_k\,\enc(\vx_{p-1}),\, p\,\boldsymbol{\theta}_R + \vpsi(\vc)\bigr) \bigr\rVert^2,
\label{eq:ctx_energy}
\end{equation}
where $\vpsi(\vc)$ reads $\vc$ as $D/2$ phasors.

We maintain the outer-product matrix $\mG = \sum_{e} \hat{\vc}_e \hat{\vc}_e^{\top}$ of the stored context vectors, where $\hat{\vc}_e = \vc_e/\lVert\vc_e\rVert$, a $D \times D$ matrix, and their mean $\bar{\vc}$. The context latent starts at $\bar{\vc}$, is carried from transition to transition, and is settled by Eq.~\ref{eq:ctx_energy} with updates given by $\vc \gets \vc - 2\,\mG\,\nabla_{\vc} E_p(\vc)$, which keeps the inferred context in the span of the stored context vectors. A new transition gets $64$ settling steps, followed by one sweep over the previously observed transitions with $4$ settling steps per transition, so that the context latent integrates the evidence of the whole prefix instead of updating solely based on whichever transition it saw last. The episode is then recalled by the decode relay under the settled context latent and scored token by token. The state never grows with $K$ or $L$: the context latent, the outer-product matrix, and the mean are $D + D^2 + D$ numbers.

\paragraph{The ideal observer.} The reference knows which stored episodes are consistent with the revealed prefix and commits to one of them uniformly at random. Its expected recall is the mean token overlap between the consistent episodes and the target. On nested examples it is an upper bound.

\paragraph{Results.} The top panel of Figure~\ref{fig:context_inference_heatmap} shows a single example, and the bottom panel shows the mean recall over 96 evaluation examples. In particular, we use $32$ $8$-episode examples from models trained under 3 seeds (i.e., $32\cdot3=96$ evaluation examples). Under the true context vector every episode is recalled exactly. Under a zero context vector recall is at the floor, at $0.093$ for nested examples.

\subsection{Evaluation details}
\label{app:eval_details}
Reported scores are computed on the test split for each task. Models are trained on a training split and early stopping is applied based on a validation set.

\paragraph{Write-rule comparison (Table~\ref{tab:write_ladder}).} The task is single-episode sequential recall at $L = D = 256$ (Appendix~\ref{app:data_seq}). We evaluate under three random seeds. Each write rule starts from an empty memory and at each position, we evaluate the feedforward key $\vk_t = R(\mW_k\,\enc(\vx_{t-1}), t)$ and value $\vv_t = \enc(\vx_t)$. The three rules are the outer product $\mS_t = \mS_{t-1} + \vv_t\vk_t^{\top}$, the delta rule $\mS_t = \mS_{t-1} + (\vv_t - \mS_{t-1}\vk_t)\vk_t^{\top}/\lVert\vk_t\rVert^2$ with a unit step, and RLS (Equation~\ref{eq:rls_update}). In the \emph{untrained} systems every learning rate is zero, so the encoder, decoder, key map and RoPE frequencies keep their initial values. The \emph{trained} systems are three seeds at $L = D = 256$, each trained from scratch with its write rule and tested with the same rule. For NLMS this removes the stored-pairs term from the energy (the NLMS ablation of Appendix~\ref{app:ablations}). The outer product rule with its keys normalized, $\mS_t = \mS_{t-1} + \vv_t\vk_t^{\top}/(\epsilon + \lVert\vk_t\rVert^2)$, recalls $0.72 \pm 0.01$ when trained from scratch. We also hold the maps trained with RLS fixed and swap only the write rule at test time, which isolates the rule from what training learned. This gives $0.42 \pm 0.05$ for the outer product and $0.51 \pm 0.02$ for NLMS, so again only RLS is exact.

Recall uses the same free-running relay of Appendix~\ref{app:data_seq}, with one change: it does not decode through $\dec$. An untrained decoder is a random map, and it will perform at chance accuracy ($0.10$) even when handed the correct stored value. So we match each memory read by cosine similarity to the nearest of the $V = 10$ encoded tokens $\enc(\vx_j)$ and feed that token's observation forward to form the next key. This clean-up step is the classical readout of a linear associative memory \citep{kohonen1972}. We use it for every system so that the comparison isolates the write rule and the learned address code. We report the mean and sample standard deviation over the three seeds.

\paragraph{Comparison with DeltaNet (Table~\ref{tab:deltanet_comparison}).} Both models see the same MQAR streams (Appendix~\ref{app:data_mqar}) at width $256$ and are scored on the $192{,}000$ test queries. Model predictions are compared to the vocabulary entries. For DeltaNet-PC this means finding the nearest codebook row to the decoded output. For DeltaNet it means the argmax of the logits.

\emph{DeltaNet.} We train the DeltaNet of \citet{yang2024c} as implemented in Zoology's code \citep{arora2024} with the \texttt{fla} kernels.\footnote{\url{https://github.com/HazyResearch/zoology}} The model has two layers, two heads of width $128$, short convolutions, no MLP, and $d_{\text{model}} = 256$. Its learned token embedding is tied to the output layer. Training follows Zoology's code: AdamW with weight decay $0.1$, a cosine schedule over $64$ epochs, batch $128$, and learning rate $4.6 \times 10^{-4}$.

\emph{DeltaNet-PC.} We use our default hyperparameters (Table~\ref{tab:hparams}) with any MQAR deviations listed in Appendix~\ref{app:hparams}. The answer at a labelled position $t$ is decoded from $\mS_{t-1}\vk_t$, which is before position $t$ is written. During testing, each write first runs a $16$-step settle of the value latent on Equation~\ref{eq:full_energy}, with the stream token clamped at the output, the key at its feedforward value, and the memory update re-solved in closed form at every step.

\emph{Training streams to exact recall.} Recall that we evaluated how few training streams are needed for exact recall. We train each model on the first $n$ of Zoology's training streams and score it on the full test set. DeltaNet still trains for $64$ epochs of $782$ batches of $128$, with batches drawn from the $n$ streams. Its test score is taken at the epoch with the best accuracy on our $128$-stream validation split.

\emph{RoPE on and off.} RoPE is off on MQAR only to match DeltaNet, whose implementation has no position code. In general, we would prefer to keep learnable RoPE on because it is the more general system that uses no task-specific information (i.e., fewer inductive biases). Table~\ref{tab:mqar_rope} shows what this choice does. With RoPE off, the untrained RLS memory retrieves every pair when each read is matched to the nearest encoded codebook row by cosine similarity (the clean-up readout of Table~\ref{tab:write_ladder}), and training mainly fits the decoder. With RoPE on, as in every other task, the untrained memory retrieves $57\%$ of pairs under the same readout, because the rotary code at initialization separates a query from the key it was written under. Training then learns to make the keys position-invariant. In the three full-data runs, $70$ to $74$ of the $128$ rotary bands end up turning less than half a radian over the stream ($31$ at initialization), and $50$ to $53$ end at exactly zero. With RoPE on, $1{,}024$ training streams give $2$ errors in $192{,}000$ queries (one seed). With our fixed codebook at its input and output, and a learned $d \times d$ map before the output in place of our decoder, DeltaNet still needs $1{,}024$ streams to be exact ($182$ errors at $256$ streams and $6$ at $512$, one seed).

\begin{table}[ht]
\centering
\small
\caption{MQAR test accuracy of DeltaNet-PC at $D = 256$ with RoPE off (the setting of Table~\ref{tab:deltanet_comparison}) and on (as in every other task). \emph{Clean-up}: each memory read is matched to the nearest encoded codebook row by cosine similarity, with no settle, as in Table~\ref{tab:write_ladder}. \emph{Decoder}: the protocol of Table~\ref{tab:deltanet_comparison}, with the $16$-step test-time settle. Untrained and full-data cells are means over $3$ seeds, and the $32$-stream cells are one seed.}
\label{tab:mqar_rope}
\begin{tabular}{@{}lcccc@{}}
\toprule
 & \multicolumn{2}{c}{Untrained} & \multicolumn{2}{c}{Trained (decoder)} \\
\cmidrule(lr){2-3} \cmidrule(lr){4-5}
 & Clean-up & Decoder & $32$ streams & $100{,}000$ streams \\
\midrule
RoPE off & $1.000$ & $0.667$ & $1.000$ & $1.000$ \\
RoPE on & $0.570$ & $0.043$ & $0.964$ & $0.999$ \\
\bottomrule
\end{tabular}
\end{table}

\emph{Locality.} A rule is temporally local when no gradient crosses positions. In DeltaNet-PC, the write depends only on $\vk_t$, $\vv_t$, $\mS_{t-1}$ and $\mP_{t-1}$, and each parameter gradient depends only on the settled prediction errors at the current position. A rule is spatially local when each weight's update is the product of the activity and the prediction error at the two ends of that weight. The one exception is LayerNorm's mean and variance, which pool over the units of a layer. DeltaNet is trained by backpropagation through time and through its layers, so its rule is neither local in time nor in space.

\emph{Parameters and state.} Parameters count every trained weight. DeltaNet's $0.53$\,M excludes its $8{,}192 \times 256$ token embedding ($2.10$\,M, with $2.63$\,M in total). DeltaNet learns its embedding and DeltaNet-PC  uses a fixed codebook. DeltaNet-PC's $197{,}376$ parameters are $\mW_{\enc}$, $\mW_k$ and the decoder weight ($256 \times 256$ each), plus the decoder bias and LayerNorm. There are no RoPE frequencies because RoPE is off on MQAR. State counts fast-weight floats. For DeltaNet this is two layers times two heads of $128 \times 128$, or $65{,}536$. For DeltaNet-PC it is $\mS$ plus the RLS inverse $\mP$, or $2 \times 256^2 = 131{,}072$. The value accumulator $\mC^v$ of Equation~\ref{eq:reg_v} is an optional $D \times D$ matrix that can be used by the learn pass and the test-time settle, but is not needed (see Appendix~\ref{app:ablations}).

\emph{Peak training memory.} This is the peak GPU memory of one training step at batch $32$, on random streams in the MQAR layout with the stream length $L$ varied. For DeltaNet we record \texttt{torch.cuda.max\_memory\_allocated} over the forward pass, the backward pass and the AdamW step. We measure a leaner variant that reads our fixed codebook instead of the embedding, computes logits only at the $64$ labelled positions, and uses no activation checkpointing. Block-level activation checkpointing trims DeltaNet's peak by about a quarter, but the peak still grows linearly with $L$ and reaches $81$\,GB at $L = 131{,}072$. Truncated BPTT would bound the memory, but it gives up credit assignment across truncation boundaries and is still not local in space. For DeltaNet-PC we record JAX's peak bytes in use during training, learn pass included. The trainer processes the stream in chunks of $512$ positions and carries the $D \times D$ state across chunks. Both models use the $8{,}192$-token vocabulary at every $L$.

\subsection{Compute resources}
Experiments ran on a combination of a single Nvidia RTX 4090 (24 GB) for development and primarily Nvidia H100 (80GB and 94GB variants) GPUs for multi-run experiments involving a range of configurations and 3 random seeds, though some V100s were also used. Individual runtimes range from a few minutes to about 5.5 hours depending on the size of the model and difficulty of the task, but median runtimes are about 10 minutes.

Our work is built on JAX \citep{jax2018github} and Equinox \citep{kidger2021equinox}, and DeltaNet-PC runs in single precision.

\section{Hyperparameters}
\label{app:hparams}

Table~\ref{tab:hparams} lists the default hyperparameter values used throughout the paper. We made an effort to use a single set of hyperparameters across all experiments, not only to demonstrate generality of the recipe, but also to ease future development efforts for any researchers who wish to build on this work.

\begin{table}[h]
\centering
\small
\footnotesize\setlength{\tabcolsep}{4pt}
\begin{tabular}{@{}llrp{6.0cm}@{}}
\toprule
\textbf{Symbol} & \textbf{Name} & \textbf{Value} & \textbf{Description} \\
\midrule
\multicolumn{4}{@{}l}{\emph{Task}} \\
$L$ & Sequence length & 256 & Positions per sequence (swept in Figure~\ref{fig:single_ep_seq}). \\
$V$ & Vocabulary size & 10 & Distinct tokens, each a fixed Gaussian code of dimension $d_{\text{obs}}$. 8,192 for MQAR. \\
$d_{\text{obs}}$ & Observation dimension & 256 & Codebook row dimension (images: pixel count). \\
$n_{\text{train}}$ & Training sequences & 1024 & Unique training sequences. \\
$n_{\text{test}}$ & Test sequences & 128 & Held-out sequences. The validation split is the same size. \\
$B$ & Batch size & 32 & Sequences per parameter update. \\
\midrule
\multicolumn{4}{@{}l}{\emph{Architecture}} \\
$D$ & Latent and memory width & 256 & Dimension of $\vk_t$ and $\vv_t$, and the side of $\mS_t, \mP_t \in \mathbb{R}^{D \times D}$. \\
$\mW_{\enc}$ & Encoder & 1 & Linear layers of $\enc$ (bias-free, $\mathbb{R}^{d_{\text{obs}}} \to \mathbb{R}^{D}$). \\
$\mW_{\dec}$ & Decoder layers & 1 & LayerNorm then one linear layer $\mathbb{R}^{D} \to \mathbb{R}^{d_{\text{obs}}}$. \\
$\mW_k$ & Key map init & identity & Initialization of the square key map $\mW_k$ (the address starts at $\enc$'s geometry). \\
$b$ & RoPE base (init) & 10000 & Initialization of the $D/2$ learnable frequencies, $\theta_i = b^{-2(i-1)/D}$. \\
\midrule
\multicolumn{4}{@{}l}{\emph{Memory write}} \\
$\epsilon$ & Ridge & 0.25 & Prior precision on $\mS$ (Eq.~\ref{eq:joint_dist}), the RLS ridge. \\
$T_{\text{write}}$ & Settling steps at the write & 0 & The write uses the feedforward $\vk_t$, $\vv_t$ (exact minimizer of Eq.~\ref{eq:energy_t}). \\
$T_{\text{write}}^{\text{test}}$ & Settling steps at the write (test) & 0 & Same at evaluation (a 16-step test settle only where noted). \\
\midrule
\multicolumn{4}{@{}l}{\emph{Learn pass}} \\
$T$ & Settling steps & 16 & Gradient steps per position on the learn-pass energy, over $(\tilde{\vk}_t, \tilde{\vv}_t)$. \\
$\eta_x$ & Settling step size & 0.02 & Gradient-descent step on the latents during settling. \\
-- & Activity gradient clip & 2 & Bound on each component of the latent gradient (elementwise, so each unit clips its own error). \\
$\lambda_{\text{ev}}$ & Evidence precision & 1 & Weight of $\tfrac12\lVert\vv_t - \enc(\vx_t)\rVert^2$, which trains $\enc$. \\
$\lambda_{k}$ & Address-prior precision & 1 & Weight of $\tfrac12\lVert\tilde{\vk}_t - R(\mW_k\,\enc(\vx_{t-1}), t)\rVert^2$, which trains $\mW_k$ and $R$. \\
\midrule
\multicolumn{4}{@{}l}{\emph{Regularizers}} \\
$\lambda_v$ & Value orthogonality & 10 & On $\mathcal{R}^{\vv}_t$ (Eq.~\ref{eq:reg_v}). \\
\midrule
\multicolumn{4}{@{}l}{\emph{Parameter learning}} \\
$\eta_\theta$ & Learning rate ($\enc$, $\dec$) & 0.00025 & Gradient descent on $\mW_{\enc}$, $\mW_{\dec}$ and LayerNorm, one step per batch of $B$. \\
$\eta_k$ & Learning rate ($\mW_k$) & 0.05 & SGD on the key map. \\
$\eta_R$ & Learning rate (RoPE) & 0.001 & SGD on the frequencies, projected onto $[0, \pi]$ after each step. \\
$L_{\text{ref}}$ & Position-mean reference & 256 & Every parameter gradient ($\mW_{\enc}$, $\mW_{\dec}$, LayerNorm, $\mW_k$, $R$) is the mean over a sequence's queries times $L_{\text{ref}}$, then its group's rate applies. \\
$D_{\text{ref}}$ & Width-scale reference & 256 & The gradient of $\mW_{\dec}$ alone is further scaled by $\min(1, (D_{\text{ref}}/D)^{p})$. \\
$p$ & Width-scale power & 2 & Exponent of the width scale ($\mW_{\dec}$ only). \\
\midrule
\multicolumn{4}{@{}l}{\emph{Early stopping}} \\
-- & Metric & recall\_accuracy & Validation recall (binding accuracy on the key-value rows). \\
-- & Interval & 1 & Every $N$ epochs (every 32 batches on the official MQAR). \\
-- & Patience & 25 & Evaluations without improvement (cap 400 epochs). The best checkpoint is tested. \\
\bottomrule
\end{tabular}

\caption{DeltaNet-PC default hyperparameters.}
\label{tab:hparams}
\end{table}

Here we list any deviations from the default hyperparameters.

\paragraph{MQAR.} MQAR uses 64 settling steps on the learn pass during training. At test-time, it uses 16 settling steps on the value latent and memory writes when updating the memory, in contrast to the other experiments that perform no settling when writing to memory at test time. We disable RoPE for MQAR, consistent with the Zoology DeltaNet implementation.

\paragraph{Omniglot.} Omniglot uses 64 settling steps on the learn pass during training.

\section{Ablations}
\label{app:ablations}

We remove one component of the full recipe at a time and retrain from scratch on each task, holding everything else fixed. Each ablation is a single seed, and the full recipe row is the task's reference. Table~\ref{tab:ablations} reports test accuracy.

\begin{table}[H]
\centering
\small
\caption{Ablations of several system components---each row removes one component from the full recipe. Test accuracy on sequential recall ($L = D = 256$), multi-episode recall ($K = 8 \times 32$), MQAR ($D = 256$) and Omniglot-MNIST binding with $128$ keys. On Omniglot the full recipe makes $2$ errors in $16{,}384$ probes and the variant without the clip makes $1$. The official MQAR row runs with RoPE off, so freezing it does not apply there.}
\label{tab:ablations}
\resizebox{\ifdim\width>\textwidth\textwidth\else\width\fi}{!}{%
\begin{tabular}{@{}lcccc@{}}
\toprule
 & Sequential & Multi-episode & MQAR & Omniglot \\
\midrule
Full recipe & $1$ & $1$ & $1$ & $0.9999$ \\
No activity clip & $1$ & $1$ & $1$ & $0.9999$ \\
No value orthogonality ($\lambda_v = 0$) & $1$ & $1$ & $1$ & $0.926$ \\
NLMS write (no interference term) & $0.282$ & $0.178$ & $0.978$ & $0.807$ \\
Frozen RoPE & $0.359$ & $0.689$ & -- & $0.940$ \\
\bottomrule
\end{tabular}}
\end{table}

\paragraph{No activity clip.} The settling gradient is not clipped. The recipe clips each of its components to magnitude $2$. Removing it has no effect on task performance. We nevertheless keep the clip to safeguard against loss spikes.
\paragraph{No value orthogonality.} $\lambda_v = 0$, which removes $\mathcal{R}^{\vv}_t$ and its leave-same-item-out correction from the energy (\S\ref{app:value_ortho}). This regularizer has no impact on the token-based tasks but results in an approximately $7$-point reduction in accuracy on Omniglot, consistent with the idea that image data is more correlated than the token codebook rows. It is therefore an optional regularizer that may or may not be needed but we keep it active in all experiments for consistency.
\paragraph{NLMS write.} The stored-pairs term $\tfrac12\Tr(\Delta\mS_t\mC_{t-1}\Delta\mS_t^{\top})$ is removed from the energy. The minimizer of what remains is $\Delta\mS_t = \vr_t\vk_t^{\top}/(\epsilon + \lVert\vk_t\rVert^2)$, the normalized delta rule of Table~\ref{tab:write_ladder} with unit step and regularization constant $\epsilon$. The write keeps its ridge but no longer protects past keys. We observe a drop in accuracy on all tasks, which provides further evidence for the value of the RLS objective.
\paragraph{Frozen RoPE.} The rotary frequencies stay at their initialization, the standard RoPE schedule with base $10^4$, instead of being learned. MQAR does not use RoPE, consistent with \citet{yang2024c}'s implementation. For the remaining tasks, we observe an intuitive drop in accuracy. The single episode sequential task is heavily reliant on position because sequences are of length 256 and therefore, keys must learn to orthogonalize to avoid collisions and rotating key planes is a mechanism for achieving this. The multi-episode sequential task relies on episodes of length $\ell=32$ and is slightly less reliant on learned rotation frequencies. Finally, Omniglot is a position-invariant task, and therefore even less reliant on learning fast rotation frequencies and the standard RoPE initialization is already largely sufficient.

\section{Compute Cost}
\label{app:compute_cost}

We compare compute on MQAR, the one task on which both DeltaNet-PC and DeltaNet are implemented. Each stream has $L = 512$ tokens and $P = 64$ queries. The width is $D = 256$, the vocabulary has $V = 8{,}192$ tokens, and each token is $D$-dimensional. We count the floating-point operations for one stream, counting a multiply-add as two. For DeltaNet-PC we sum the operations of Algorithm~\ref{alg:training} over the positions of the stream. For DeltaNet we count the recurrent form of its two layers and take the backward pass to cost twice the forward pass. Table~\ref{tab:compute} summarizes the comparison.

\begin{table}[H]
\centering
\small
\setlength{\tabcolsep}{3.5pt} %
\caption{Compute per MQAR stream ($L = 512$, $P = 64$ queries, $D = 256$, $V = 8{,}192$, $T = 64$) and measured time on one H100. DeltaNet-PC's test figures in parentheses include the $16$-step test-time settle, e.g., 0.75 for no test-time settle, and 4.0 with a 16-step test-time settle. Training times are per stream at batch $32$ for DeltaNet-PC and $128$ for DeltaNet. Streams to exact recall count every training stream the model observed, including repeats.}
\label{tab:compute}
\begin{tabular}{@{}lll@{}}
\toprule
 & DeltaNet-PC & DeltaNet \\
\midrule
Training, operations & $O(LD^2 + PTD^2)$, $3.2$ GFLOP & $O(L(D^2 + DV))$, $8.7$ GFLOP \\
Test, operations & $O(LD^2 + PDV)$, $0.75$ ($4.0$) GFLOP & $O(L(D^2 + DV))$, $2.9$ GFLOP \\
Training, time & $18$ ms & $0.42$ ms \\
Test, time & $3$ ($12$--$20$) ms & $0.8$ ms \\
To exact recall & $1{,}024$ observed streams, $3.3$ TFLOP, $18$ s & $200{,}192$ observed streams, $1.7$ PFLOP, $88$ s \\
\bottomrule
\end{tabular}
\end{table}

\paragraph{Operations.} In training, DeltaNet-PC writes every position with the RLS update, at $O(D^2)$ per position. At each of the $P$ queries it also runs a $T$-step learn pass, in which each step is a memory read, a decoder pass and the gradient of that pass, for $O(TD^2)$ per query. A training stream therefore costs $O(LD^2 + PTD^2)$, and with $T = 64$ the learn pass dominates. DeltaNet costs $O(D^2)$ per token in each layer, plus $O(DV)$ per token for its output logits, and backpropagation through time multiplies this by a constant, for $O(L(D^2 + DV))$ per stream. At test time, DeltaNet-PC writes every position at $O(D^2)$ and decodes each query against the $V$ codebook rows at $O(DV)$, for $O(LD^2 + PDV)$ per stream. The $16$-step test-time settle adds $O(16\,LD^2)$. DeltaNet runs its forward pass, at $O(L(D^2 + DV))$. Per stream, the two models are within a small factor of each other in both training and test. The large difference is in how many training streams each model needs. DeltaNet-PC is exact after $32$ updates of $32$ streams, about $3$ TFLOP in total. DeltaNet's fastest route to exact recall is two epochs of the $100{,}000$ training streams, which requires about $1.7$ PFLOP.

\paragraph{Wall-clock.} The operation counts do not carry over to wall-clock time on a GPU. On one H100, DeltaNet-PC trains at $18$ ms per stream and DeltaNet at $0.42$ ms per stream, even though a DeltaNet-PC training stream takes fewer operations. Our implementation scans the stream one position at a time, and at each query the learn pass is a chain of $T$ dependent matrix-vector products on a batch of $32$.
The time to exact recall is nonetheless shorter for DeltaNet-PC: $18$ s of training for DeltaNet-PC and $88$ s for DeltaNet. These times describe our current implementation and future work could improve our DeltaNet-PC code to be more parallel (e.g., chunked updates).

\section{tPC reference}
\label{app:tpc_reference}

The tPC capacity curve was produced by training tPC at $L=8$, $D=256$, with 
a unique $\mathrm{BOS}$ token per training sequence. This is tPC's most favorable condition for memorization: each sequence has its own indexing token, so the slow weights can in principle dedicate distinct capacity to each. We were unable to reliably train tPC at longer $L$. The plot shows recall degrading as $N$ grows, exposing the slow-weight capacity ceiling. We do not present this as a direct comparison against DeltaNet-PC: the two systems target different problems---tPC memorizes a fixed set of training sequences in the slow weights and DeltaNet-PC binds arbitrary unseen sequences into $\mS$ at test time.

\begin{figure}[htbp]
    \centering
    \includegraphics{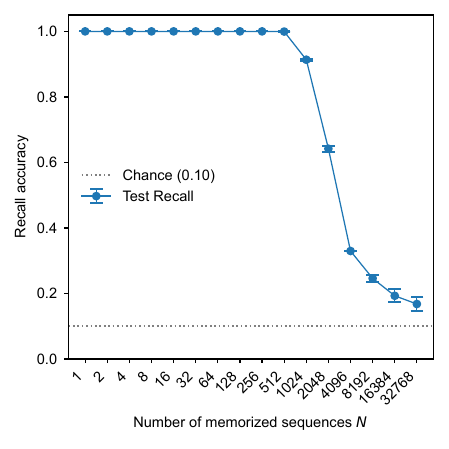}
    \caption{tPC's intrinsic memorization ceiling. Test recall accuracy on the $N$ training sequences as a function of $N$, with a unique $\mathrm{BOS}$ token per training sequence. $L=8$, $D=256$.}
    \label{fig:tpc_capacity}
\end{figure}

\end{document}